%% file: Main.tex
\documentclass{article} 

\usepackage{rjm,times}

\input{math_commands.tex}

\usepackage{amsmath,amssymb}
\usepackage{amsthm}

\newtheorem{theorem}{Theorem}[section]
\newtheorem{proposition}[theorem]{Proposition}
\newtheorem{corollary}[theorem]{Corollary}

\usepackage{booktabs}
\usepackage{multirow}
\usepackage{array}
\usepackage{xcolor}
\usepackage{colortbl}
\usepackage{graphicx}
\usepackage{url}
\usepackage{hyperref}
\usepackage{float}

\usepackage[ruled,vlined,linesnumbered]{algorithm2e}
\SetAlgoNlRelativeSize{-1}
\DontPrintSemicolon

\definecolor{oursbg}{RGB}{230,245,236}

\newcommand{\best}[1]{\mathbf{#1}}
\newcommand{\second}[1]{\underline{#1}}

\title{
AlphaRJM: Reward-Jump Memory for\\
Stochastic Return-Guided Alpha Discovery
}

\author{
Sayan Dhan
\qquad\qquad
Selvaraju Natarajan
\\[5pt]
Department of Mathematics\\
Indian Institute of Technology Guwahati\\
[3pt]
\texttt{s.dhan@iitg.ac.in}
\qquad
\texttt{nselvaraju@iitg.ac.in}
}

\iclrfinalcopy

\begin{document}

\maketitle

\fancyhead{}
\renewcommand{\headrulewidth}{0pt}

\begin{abstract}
Formulaic alpha discovery is a pool-dependent symbolic search problem in
which informative feedback is observed primarily when a complete expression
is evaluated. This delayed feedback creates two coupled difficulties: the
retained alpha pool does not preserve the full history of realized evaluation
feedback, and the value of an intermediate construction action is uncertain
because its consequence depends on the formula eventually completed. We
introduce \textbf{AlphaRJM}, which addresses these difficulties through
\textbf{Reward-Jump Memory}, an event-driven latent state that remains fixed
during token construction and updates only at terminal evaluation events
using the realized pool reward and evaluation outcome, and an
\textbf{action-conditioned SDE return critic} that represents future
discounted discovery returns with stochastic particles. The particles guide
action selection through their mean and uncertainty and are learned using a
distributional Bellman objective combining energy-distance matching, mean
calibration, and jump regularization. Empirically, AlphaRJM delivers strong
and stable gains across multiple equity universes, forecasting horizons, and
random seeds, while ablations confirm the complementary roles of persistent
evaluation history, stochastic return modeling, and distributional
supervision.
\end{abstract}

\section{Introduction}
\label{sec:introduction}

Formulaic alphas are interpretable symbolic expressions that transform
historical market variables into predictive signals. Automating their
discovery is challenging because the space of candidate expressions is
combinatorial and, more importantly, the usefulness of a candidate depends
on the alpha pool already constructed. A moderately predictive formula that
contributes complementary information may be more valuable than a stronger
but redundant one. AlphaGen formalizes this interaction by rewarding a
candidate according to its marginal contribution to the combined alpha
ensemble~\citep{yu2023generating}, turning formula discovery into a
pool-dependent sequential decision problem with sparse and evolving
feedback.

Recent methods improve different aspects of this search process. AlphaQCM
introduces distributional reinforcement learning and uncertainty-aware
exploration through learned return quantiles and Quantile Conditional Moments
(QCM)~\citep{zhu2025alphaqcm,zhang2026quantiled}. AlphaSAGE instead
strengthens expression representation and search diversity through
structure-aware graph encoding, Generative Flow Networks (GFlowNets), and
dense multi-faceted rewards~\citep{chen2026alphasage}. These advances improve
uncertainty modeling, structural understanding, and exploration, but they do
not explicitly maintain a persistent representation of the sequence of
terminal rewards and evaluation outcomes produced during search.

This distinction matters because terminal feedback can be informative even
when the retained pool remains unchanged. For example, a rejected candidate
provides a different evaluation outcome from an accepted or replaced one,
while an invalid construction produces another distinct terminal event.
Consequently, two search trajectories may arrive at the same retained pool
after different sequences of realized rewards and categorical outcomes.
Conditioning subsequent decisions only on the current pool therefore discards
part of the evaluation history accumulated during search. This motivates a
persistent state that evolves with terminal evaluation feedback while
remaining unchanged during intermediate token construction.

Delayed terminal evaluation creates a second, closely related difficulty.
At an intermediate token step, the consequence of selecting an action is not
observed immediately: its utility depends on the expression that will
eventually be completed and on the pool against which that expression will be
evaluated. The future discovery return associated with an action is therefore
naturally uncertain rather than a single deterministic quantity. A
distributional representation can preserve this uncertainty and provide both
a central estimate and a dispersion signal for action selection.

To address these two coupled challenges, we propose
AlphaRJM, a formulaic alpha-discovery framework that combines
Reward-Jump Memory with an
action-conditioned stochastic differential equation (SDE) return
critic. Reward-Jump Memory is an event-driven persistent latent state that
remains unchanged during token-level formula construction and updates only
when a terminal evaluation event occurs. Each update is driven by the
realized pool reward and categorical evaluation outcome, allowing this
evaluation history to persist across formula boundaries independently of
whether the retained pool changes.

Complementing this persistent history, the action-conditioned SDE critic
models the uncertain future discovery return of each candidate construction
action. Conditioned on the encoded partial expression, encoded alpha pool,
persistent memory, and candidate action, the critic defines a
state--action-specific stochastic return process. Independent trajectories
generate particles representing plausible future discounted discovery
returns, whose empirical mean and uncertainty directly guide action
selection. The SDE is used solely as an internal stochastic return critic for
symbolic search rather than as a model of asset prices or alpha signals. Its
particle distribution is trained through a distributional Bellman objective
combining energy-distance matching, mean calibration, and jump
regularization.

Empirically, AlphaRJM achieves strong and stable gains across multiple equity
universes, forecasting horizons, and random seeds, while controlled ablations
support the complementary roles of persistent evaluation history, stochastic
return modeling, and distributional supervision.

\paragraph{Contributions.}
Our main contributions are:
\begin{itemize}

    \item We introduce \textbf{Reward-Jump Memory}, an event-driven
    persistent state that updates only after terminal evaluations using the
    realized pool reward and categorical outcome, allowing evaluation history
    to persist across formula episodes.

    \item We develop an \textbf{action-conditioned SDE return critic} that
    represents uncertain future discovery returns with stochastic particles
    and uses their mean and uncertainty to guide symbolic action selection.

    \item We train the resulting framework through
    \textbf{distributional Bellman learning} with sequence-consistent memory
    replay, and demonstrate its effectiveness across multiple equity
    universes, forecasting horizons, and random seeds.

\end{itemize}

\section{Background and Related Work}
\label{sec:background}

Formulaic alpha discovery searches for interpretable symbolic expressions
whose usefulness depends on the existing alpha pool. AlphaGen formalizes this
interaction by rewarding a candidate according to its marginal contribution
to the combined ensemble~\citep{yu2023generating}. AlphaRJM adopts this
pool-dependent setting while additionally maintaining a persistent
representation of realized terminal evaluation feedback across successive
formula episodes.

Distributional reinforcement learning models the distribution of discounted
returns rather than only their expectation
~\citep{bellemare2017distributional,dabney2018implicit}. AlphaQCM applies
this principle to formulaic alpha discovery through learned return quantiles
and Quantile Conditional Moments for uncertainty-aware exploration
~\citep{zhu2025alphaqcm,zhang2026quantiled}. AlphaRJM retains the
distributional perspective but conditions future returns on the current
expression, alpha pool, persistent evaluation history, and candidate action,
and represents the resulting return law using particles generated by an
action-conditioned SDE.

A complementary line of work improves symbolic representation and
exploration. AlphaSAGE combines structure-aware relational graph
representations with Generative Flow Networks (GFlowNets) and dense
multi-faceted rewards to improve expression modeling and search diversity
~\citep{chen2026alphasage,schlichtkrull2018modeling,
bengio2021flow,malkin2022trajectory}. AlphaRJM addresses a different aspect
of the search process: retaining realized reward--outcome information across
formula boundaries and using it to condition subsequent distributional
action evaluation.

Neural SDEs provide flexible mechanisms for representing and sampling
stochastic dynamics~\citep{kidger2021neural}, while energy-distance and
proper-scoring objectives support sample-based distribution learning
~\citep{gneiting2007strictly,szekely2013energy}. In AlphaRJM, the SDE is used
solely as an internal stochastic return critic rather than as a model of asset
prices or alpha signals.

Overall, AlphaRJM complements prior work on pool-dependent rewards,
distributional exploration, and structure-aware symbolic generation by
coupling persistent terminal evaluation history with a sampleable
distribution of future discovery returns.

\section{Methodology}
\label{sec:methodology}

AlphaRJM augments pool-dependent symbolic alpha discovery with two coupled
components: an event-driven persistent memory that carries terminal
evaluation feedback across formula episodes, and an action-conditioned
stochastic return critic that represents uncertain future discovery returns.
We first formalize the symbolic search state and terminal-event structure,
then introduce Reward-Jump Memory, stochastic return-guided action selection,
and distributional Bellman learning.

\subsection{Framework and Search State}
\label{sec:framework}

We follow the grammar-based symbolic construction setting used in prior
formulaic alpha-discovery methods
~\citep{yu2023generating,zhu2025alphaqcm,chen2026alphasage}.
Let $t$ index primitive token-level decisions and let
$n\in\{0,\ldots,N-1\}$ denote the number of completed formula episodes
before the current decision, where $N$ is the total formula-generation
budget.

At step $t$, the search context contains a partial expression $x_t$ and
the current retained alpha pool $\mathcal P_t$. Let
$\mathcal A_t$ denote the set of actions exposed by the current symbolic
construction mask. An action
\begin{equation}
    a_t\in\mathcal A_t
\end{equation}
appends an operator, market feature, temporal token, constant, or terminal
symbol to the current expression. The exact grammar and construction-mask
protocol are given in Appendix~\ref{app:symbolic_search}.

Let $\delta_t\in\{0,1\}$ indicate whether action $a_t$ produces a terminal
formula evaluation. Thus, $\delta_t=0$ denotes an intermediate construction
step and $\delta_t=1$ terminates and evaluates the current formula episode.
In the reported environment every terminated formula attempt produces such
an evaluation event, so $\delta_t$ also serves as the formula-episode
termination indicator.

At terminal events, let $o_t$ denote the realized categorical evaluation
outcome, with
\begin{equation}
    o_t\in
    \mathcal O
    :=
    \{
    \mathrm{accepted},
    \mathrm{replaced},
    \mathrm{rejected},
    \mathrm{invalid}
    \}.
    \label{eq:outcome_set}
\end{equation}
The outcome is irrelevant when $\delta_t=0$.

Following the synergistic pool formulation of
AlphaGen~\citep{yu2023generating}, the reward is the marginal change in
training-set ensemble Information Coefficient (IC):
\begin{equation}
    r_t
    =
    \begin{cases}
    Q(\mathcal P_{t+1})-Q(\mathcal P_t),
        & \delta_t=1,\\
    0,
        & \delta_t=0,
    \end{cases}
    \label{eq:pool_reward}
\end{equation}
where $Q(\mathcal P)$ denotes the IC of the weighted alpha ensemble
represented by pool $\mathcal P$. Hence the same candidate may have different
utility under different retained pools.

Let $E_{\phi}$ denote the partial-expression encoder with parameters $\phi$
and $P_{\psi}$ the permutation-invariant pool encoder with parameters
$\psi$. Their outputs are
\begin{equation}
    e_t=E_{\phi}(x_t),
    \qquad
    p_t=P_{\psi}(\mathcal P_t),
    \label{eq:encodings}
\end{equation}
where $e_t$ represents the current expression and $p_t$ the retained pool.
AlphaRJM augments these observable representations with a persistent latent
memory $H_t$ and defines the search state
\begin{equation}
    s_t=(e_t,p_t,H_t).
    \label{eq:rjm_state}
\end{equation}
The two history-bearing components have distinct roles: $p_t$ summarizes
the formulas currently retained in the ensemble, whereas $H_t$ summarizes
the sequence of realized reward--outcome feedback from previous terminal
evaluations.

\subsection{Reward-Jump Memory}
\label{sec:rjm}

At the beginning of the search stream, the persistent state is initialized
from the initial expression and pool representations:
\begin{equation}
    H_0
    =
    \tanh\!\left(
        W_0[e_0;p_0]+b_0
    \right).
    \label{eq:rjm_init}
\end{equation}
Here $[\cdot\,;\cdot]$ denotes vector concatenation. The implementation
details of this one-time initialization are provided in
Appendix~\ref{app:rjm_details}.

For a terminal event, the realized reward and outcome are encoded into an
event mark
\begin{equation}
    \chi_t
    =
    \mathcal M_{\eta}(r_t,o_t),
    \label{eq:rjm_mark}
\end{equation}
where $\mathcal M_{\eta}$ is the learned event encoder. The event mark and
current memory are then transformed into a bounded candidate jump:
\begin{align}
    u_t
    &=
    \mathcal C_{\eta}([H_t;\chi_t]),
    \nonumber\\
    g_t
    &=
    \operatorname{sigmoid}(W_g u_t+b_g),
    \qquad
    \zeta_t
    =
    \tanh(W_{\zeta}u_t+b_{\zeta}),
    \nonumber\\
    J_t
    &=
    J_{\max}\,
    g_t\odot\zeta_t ,
    \label{eq:rjm_jump}
\end{align}
where $\mathcal C_{\eta}$ is the learned jump-context network,
$g_t$ is an element-wise gate, $\zeta_t$ is a signed direction vector,
$\odot$ denotes the Hadamard product, and $J_{\max}>0$ bounds the magnitude
of every jump coordinate. We use $\eta$ collectively for the learned
parameters of the event encoder, jump-context network, and jump heads.

The persistent state evolves according to
\begin{equation}
    H_{t+1}
    =
    \begin{cases}
        H_t+J_t, & \delta_t=1,\\
        H_t,     & \delta_t=0.
    \end{cases}
    \label{eq:rjm_update}
\end{equation}
Thus $H_t$ is exactly constant throughout intermediate token construction
and changes only after terminal evaluation feedback. Importantly, completion
of a formula ends its Bellman episode but not the memory stream:
the post-event state $H_{t+1}$ is retained when construction of the next
formula begins. This produces a fast token-construction timescale and a
slower event-driven memory timescale.

\subsection{Stochastic Return-Guided Alpha Discovery}
\label{sec:stochastic_return}

Reward-Jump Memory carries past evaluation feedback, while action selection
requires estimating the uncertain future consequence of each available
construction action. AlphaRJM represents this uncertainty through an
action-conditioned scalar SDE return critic.

For candidate action $a$, let $A_{\omega}(a)$ denote its learned action
embedding with parameters $\omega$. The critic conditioning vector is
\begin{equation}
    q_{t,a}
    =
    [e_t;p_t;H_t;A_{\omega}(a)].
    \label{eq:sde_context}
\end{equation}
A neural conditioning network maps $q_{t,a}$ to four scalars:
an initial return state $z_{0,t,a}$, a long-run level $\mu_{t,a}$,
a positive mean-reversion rate $\kappa_{t,a}$, and a bounded diffusion scale
$\sigma_{t,a}\in[\sigma_{\min},\sigma_{\max}]$, where
$0<\sigma_{\min}\leq\sigma_{\max}<\infty$.
For compactness, $\theta$ denotes the full online SDE-critic parameter
collection, including the action-embedding parameters $\omega$.

These quantities define the conditional return process
\begin{equation}
    dZ_{\tau}^{t,a}
    =
    \kappa_{t,a}
    \bigl(
        \mu_{t,a}-Z_{\tau}^{t,a}
    \bigr)\,d\tau
    +
    \sigma_{t,a}\,dB_{\tau},
    \qquad
    Z_0^{t,a}=z_{0,t,a},
    \label{eq:return_sde}
\end{equation}
where $\tau\in[0,1]$ is an internal diffusion coordinate and
$B_{\tau}$ is a standard one-dimensional Brownian motion. The scalar
mean-reverting form provides an inexpensive sampleable return law with
separately controlled location, reversion, and stochastic dispersion.
Importantly, $Z_{\tau}^{t,a}$ is not a stock-price process or an alpha
signal; it represents the stochastic discounted reinforcement-learning
return associated with choosing action $a$ in search state $s_t$.
Its conditional law is characterized in Appendix~\ref{app:theory}.

Simulating $M$ independent trajectories of
Eq.~\ref{eq:return_sde} and taking their terminal values at
$\tau=1$ yields the particle representation
\begin{equation}
    \mathcal Z_{\theta}(s_t,a)
    =
    \left\{
        Z_{t,a}^{(1)},
        \ldots,
        Z_{t,a}^{(M)}
    \right\},
    \label{eq:particles}
\end{equation}
where $Z_{t,a}^{(m)}$ denotes the terminal value of trajectory $m$.
The implementation uses Euler--Maruyama integration as detailed in
Appendix~\ref{app:sde_details}.

The empirical mean and variance of these particles are
\begin{equation}
    \overline Z_{t,a}
    =
    \frac{1}{M}
    \sum_{m=1}^{M}
    Z_{t,a}^{(m)},
    \qquad
    \widehat V_{t,a}
    =
    \frac{1}{M}
    \sum_{m=1}^{M}
    \left(
        Z_{t,a}^{(m)}
        -
        \overline Z_{t,a}
    \right)^2 .
    \label{eq:particle_statistics}
\end{equation}
AlphaRJM converts these statistics into the uncertainty-aware action score
\begin{equation}
    S_t(a)
    =
    \overline Z_{t,a}
    +
    c_n
    \sqrt{
        \widehat V_{t,a}
        +
        \varepsilon_{\mathrm{num}}
    },
    \label{eq:action_score}
\end{equation}
where $\varepsilon_{\mathrm{num}}>0$ is a numerical stabilizer and
$c_n\geq0$ is the uncertainty coefficient at formula episode $n$.
The coefficient decreases during search, placing greater emphasis on
uncertain actions early and increasingly emphasizing expected return later.

Action selection additionally uses an episode-dependent
$\epsilon_n$-greedy rule:
\begin{equation}
    a_t
    =
    \begin{cases}
    \text{sample uniformly from }\mathcal A_t,
        & \text{with probability }\epsilon_n,\\[1mm]
    \displaystyle
    \arg\max_{a\in\mathcal A_t}S_t(a),
        & \text{with probability }1-\epsilon_n .
    \end{cases}
    \label{eq:action_selection}
\end{equation}
The schedules of $c_n$ and $\epsilon_n$, together with the exact
construction mask, are specified in
Appendix~\ref{app:exploration}.

\subsection{Distributional Bellman Learning}
\label{sec:distributional_learning}

For the executed action $a_t$, let
\begin{equation}
    Z_t^{(m)}
    :=
    Z_{t,a_t}^{(m)}
    \sim
    \mathcal Z_{\theta}(s_t,a_t),
    \qquad
    m=1,\ldots,M,
    \label{eq:online_particles}
\end{equation}
denote particles from the online critic. Let $\bar\theta$ denote the
parameters of a periodically synchronized target copy of the complete SDE
return critic.

After the environment transition, the next state is
\begin{equation}
    s_{t+1}
    =
    (e_{t+1},p_{t+1},H_{t+1}),
\end{equation}
where $H_{t+1}$ follows Eq.~\ref{eq:rjm_update}.
For a nonterminal transition, the greedy next action is selected using the
online stochastic score:
\begin{equation}
    a_{t+1}^{\star}
    =
    \arg\max_{a\in\mathcal A_{t+1}}
    S_{t+1}(a).
    \label{eq:next_action}
\end{equation}
The target critic independently generates
\begin{equation}
    \widetilde Z_{t+1}^{(m)}
    \sim
    \mathcal Z_{\bar\theta}
    (s_{t+1},a_{t+1}^{\star}),
    \qquad
    m=1,\ldots,M,
\end{equation}
and the distributional Bellman particles are
\begin{equation}
    Y_t^{(m)}
    =
    r_t
    +
    \gamma(1-\delta_t)
    \widetilde Z_{t+1}^{(m)},
    \label{eq:bellman_particles}
\end{equation}
where $\gamma\in[0,1)$ is the discount factor. When $\delta_t=1$, the
Bellman target therefore reduces to the realized pool reward and does not
bootstrap across the formula boundary. The updated memory nevertheless
persists into the next formula episode, separating episodic return learning
from persistent evaluation history.

We match the predicted and Bellman-target particle distributions using the
empirical one-dimensional energy distance:
\begin{equation}
\mathcal L_{\mathrm{ED}}
=
\frac{2}{M^2}\sum_{m,m'=1}^{M}\left|Z_t^{(m)}-Y_t^{(m')}\right|
-
\frac{1}{M^2}\sum_{m,m'=1}^{M}\left|Z_t^{(m)}-Z_t^{(m')}\right|
-
\frac{1}{M^2}\sum_{m,m'=1}^{M}\left|Y_t^{(m)}-Y_t^{(m')}\right|.
\label{eq:energy}
\end{equation}
Energy distance directly compares distributions through samples
~\citep{szekely2013energy}; in one dimension it is closely related to
CRPS/energy-score objectives used in probabilistic prediction
~\citep{gneiting2007strictly}. We therefore refer to
Eq.~\ref{eq:energy} as a CRPS-style distributional objective.

Because finite particle sets can introduce sampling noise in the predicted
central tendency, we additionally use
\begin{equation}
    \mathcal L_{\mathrm{mean}}
    =
    \operatorname{Huber}
    \left(
        \frac{1}{M}\sum_{m=1}^{M}Z_t^{(m)},
        \frac{1}{M}\sum_{m=1}^{M}Y_t^{(m)}
    \right),
    \label{eq:mean_huber}
\end{equation}
where $\operatorname{Huber}$ denotes the Huber loss.

We further regularize the magnitude of memory jumps at actual evaluation
events. Let $\mathcal L_{\mathrm{jump}}$ denote the event-masked
squared-$\ell_2$ jump penalty; the exact replay-minibatch aggregation used
by the implementation is given in
Eq.~\ref{eq:jump_loss_exact} of
Appendix~\ref{app:replay}. The complete objective is
\begin{equation}
    \mathcal L
    =
    \lambda_{\mathrm{ED}}
    \mathcal L_{\mathrm{ED}}
    +
    \lambda_{\mathrm{mean}}
    \mathcal L_{\mathrm{mean}}
    +
    \beta_J
    \mathcal L_{\mathrm{jump}},
    \label{eq:total_loss}
\end{equation}
where $\lambda_{\mathrm{ED}}\geq0$,
$\lambda_{\mathrm{mean}}\geq0$, and $\beta_J\geq0$ control the three
loss components.

Finally, the persistence of $H_t$ makes independently shuffled transition
replay unsuitable: arbitrary reordering would break the event sequence that
defines the memory trajectory. AlphaRJM therefore trains on contiguous
transition sequences. Each sampled sequence starts from its stored
pre-transition memory, uses a burn-in prefix to reconstruct the persistent
trajectory under the current model, and applies the return-learning objective
to the subsequent unroll. The expression encoder, pool encoder,
Reward-Jump Memory, action embedding, and online SDE return critic are
optimized jointly, while the target SDE critic is synchronized periodically.
Full architecture, discretization, exploration, replay, and optimization
details are given in Appendix~\ref{app:implementation}; theoretical
properties are given in Appendix~\ref{app:theory}.

\section{Experiments and Results}\label{sec:experiments}

\subsection{Experiment Setting}
\textbf{Evaluation Metrics.}
Following established formulaic alpha discovery protocols~\citep{yu2023generating,zhu2025alphaqcm,chen2026alphasage}, we evaluate predictive performance using four correlation-based metrics: Information Coefficient (IC), Information Coefficient Information Ratio (ICIR), Rank Information Coefficient (RIC), and Rank Information Coefficient Information Ratio (RICIR). Higher values indicate better performance for all metrics. Detailed definitions and evaluation settings are provided in Appendix~\ref{app:metrics}.

\textbf{Datasets.}
We evaluate all methods on three representative Chinese equity universes: CSI300, CSI500, and CSI800. For all datasets, we use a common chronological split, with January 1, 2010--December 31, 2020 for training, January 1, 2021--December 31, 2021 for validation, and January 1, 2022--December 31, 2024 for testing. This fixed split is used consistently across all methods and random seeds. Additional protocol details are provided in Appendix~\ref{app:protocol}, and the AlphaRJM configurations are reported in Appendix~\ref{app:rjm_presets}.
We use $h$ to denote the forecasting horizon; the standard benchmark uses
$h=20$, while the long-horizon experiment uses $h=42$.

\textbf{Baselines.}
We compare AlphaRJM with representative baselines spanning three model families:
(1) conventional predictive models, including a multilayer perceptron (MLP)~\citep{murtagh1991multilayer} and Light Gradient Boosting Machine (LightGBM)~\citep{ke2017lightgbm};
(2) continuous-time neural models, including Neural ODE (ordinary differential equation)~\citep{chen2018neural} and Neural SDE~\citep{kidger2021neural}; and
(3) formulaic alpha discovery methods based on reinforcement learning, including AlphaGen~\citep{yu2023generating}, AlphaQCM~\citep{zhu2025alphaqcm}, and AlphaSAGE~\citep{chen2026alphasage}.
All baselines are evaluated under the same benchmark protocol whenever applicable; implementation details are provided in Appendix~\ref{app:baselines}.

\begin{table}[t]
\centering
\caption{
Performance comparison on CSI300, CSI500, and CSI800 using correlation-based
evaluation metrics. Results are reported as mean $\pm$ standard deviation
over four random seeds (0--3). Higher values are better for all metrics.
Best and second-best performances are shown in
\textbf{bold} and \underline{underlined}, respectively; rankings are determined
from the unrounded mean values.
}
\label{tab:main_correlation_results}

\footnotesize
\setlength{\tabcolsep}{3.0pt}
\renewcommand{\arraystretch}{1.12}

\begin{tabular}{llcccc}
\toprule
\textbf{Dataset}
& \textbf{Method}
& \textbf{IC}
& \textbf{ICIR}
& \textbf{RIC}
& \textbf{RICIR} \\
\midrule

\multirow{8}{*}{CSI300}
& MLP
& $0.030 \pm 0.017$
& $0.218 \pm 0.119$
& $0.035 \pm 0.024$
& $0.269 \pm 0.186$ \\

& LightGBM
& $0.024 \pm 0.003$
& $0.242 \pm 0.028$
& $0.020 \pm 0.001$
& $0.191 \pm 0.010$ \\

& Neural ODE
& $0.037 \pm 0.004$
& $0.260 \pm 0.028$
& $0.032 \pm 0.006$
& $0.214 \pm 0.040$ \\

& Neural SDE
& $0.032 \pm 0.005$
& $0.239 \pm 0.040$
& $0.026 \pm 0.005$
& $0.188 \pm 0.035$ \\

& AlphaGen
& $\second{0.041 \pm 0.011}$
& $\second{0.302 \pm 0.080}$
& $\second{0.046 \pm 0.021}$
& $\second{0.326 \pm 0.156}$ \\

& AlphaQCM
& $0.039 \pm 0.007$
& $0.276 \pm 0.051$
& $0.043 \pm 0.009$
& $0.301 \pm 0.073$ \\

& AlphaSAGE
& $0.033 \pm 0.016$
& $0.246 \pm 0.080$
& $0.037 \pm 0.017$
& $0.266 \pm 0.088$ \\

\rowcolor{oursbg}
& \textbf{AlphaRJM (ours)}
& $\best{0.042 \pm 0.007}$
& $\best{0.307 \pm 0.058}$
& $\best{0.046 \pm 0.003}$
& $\best{0.331 \pm 0.038}$ \\

\midrule

\multirow{8}{*}{CSI500}
& MLP
& $0.031 \pm 0.011$
& $0.269 \pm 0.088$
& $0.034 \pm 0.018$
& $0.299 \pm 0.162$ \\

& LightGBM
& $0.027 \pm 0.003$
& $0.311 \pm 0.027$
& $0.031 \pm 0.007$
& $0.364 \pm 0.077$ \\

& Neural ODE
& $0.029 \pm 0.002$
& $0.239 \pm 0.014$
& $0.028 \pm 0.002$
& $0.218 \pm 0.018$ \\

& Neural SDE
& $0.025 \pm 0.005$
& $0.215 \pm 0.030$
& $0.024 \pm 0.007$
& $0.186 \pm 0.055$ \\

& AlphaGen
& $0.036 \pm 0.005$
& $0.273 \pm 0.025$
& $0.041 \pm 0.007$
& $0.339 \pm 0.061$ \\

& AlphaQCM
& $0.035 \pm 0.007$
& $0.296 \pm 0.053$
& $0.042 \pm 0.012$
& $0.360 \pm 0.085$ \\

& AlphaSAGE
& $\second{0.041 \pm 0.007}$
& $\second{0.336 \pm 0.039}$
& $\second{0.059 \pm 0.016}$
& $\second{0.497 \pm 0.061}$ \\

\rowcolor{oursbg}
& \textbf{AlphaRJM (ours)}
& $\best{0.046 \pm 0.014}$
& $\best{0.379 \pm 0.159}$
& $\best{0.059 \pm 0.015}$
& $\best{0.523 \pm 0.173}$ \\

\midrule

\multirow{8}{*}{CSI800}
& MLP
& $0.015 \pm 0.007$
& $0.176 \pm 0.082$
& $0.018 \pm 0.013$
& $0.206 \pm 0.152$ \\

& LightGBM
& $0.026 \pm 0.003$
& $\second{0.305 \pm 0.035}$
& $0.025 \pm 0.004$
& $0.286 \pm 0.045$ \\

& Neural ODE
& $0.026 \pm 0.002$
& $0.237 \pm 0.024$
& $0.024 \pm 0.003$
& $0.197 \pm 0.033$ \\

& Neural SDE
& $0.028 \pm 0.003$
& $0.264 \pm 0.020$
& $0.028 \pm 0.003$
& $0.239 \pm 0.022$ \\

& AlphaGen
& $0.028 \pm 0.004$
& $0.258 \pm 0.054$
& $0.039 \pm 0.007$
& $0.344 \pm 0.088$ \\

& AlphaQCM
& $\second{0.029 \pm 0.017}$
& $0.260 \pm 0.147$
& $0.041 \pm 0.019$
& $0.332 \pm 0.150$ \\

& AlphaSAGE
& $0.025 \pm 0.020$
& $0.208 \pm 0.159$
& $\best{0.045 \pm 0.015}$
& $\best{0.396 \pm 0.082}$ \\

\rowcolor{oursbg}
& \textbf{AlphaRJM (ours)}
& $\best{0.036 \pm 0.007}$
& $\best{0.339 \pm 0.027}$
& $\second{0.043 \pm 0.007}$
& $\second{0.371 \pm 0.048}$ \\

\bottomrule
\end{tabular}
\end{table}

\begin{table}[t]
\centering
\caption{
IC and RIC performance on CSI300 and CSI500 for the
42-trading-day forecasting horizon ($h=42$) using random seed 0.
Higher values are better.
Best and second-best results are shown in
\textbf{bold} and \underline{underlined}, respectively.
}
\label{tab:h42_ic_ric}

\scriptsize
\setlength{\tabcolsep}{3.2pt}
\renewcommand{\arraystretch}{1.12}

\begin{tabular}{lcccccc}
\toprule
\multirow{2}{*}{\textbf{Method}}
& \multicolumn{2}{c}{\textbf{CSI300}}
& \multicolumn{2}{c}{\textbf{CSI500}} \\
\cmidrule(lr){2-3}
\cmidrule(lr){4-5}
& \textbf{IC}
& \textbf{RIC}
& \textbf{IC}
& \textbf{RIC} \\
\midrule

MLP
& $\best{0.046}$
& $0.056$
& $0.041$
& $0.050$ \\

LightGBM
& $0.027$
& $0.026$
& $0.031$
& $0.035$ \\

Neural ODE
& $0.014$
& $0.014$
& $0.028$
& $0.033$ \\

Neural SDE
& $0.007$
& $0.005$
& $0.019$
& $0.015$ \\

AlphaGen
& $0.021$
& $0.027$
& $0.034$
& $0.044$ \\

AlphaQCM
& $\second{0.045}$
& $\second{0.062}$
& $\second{0.044}$
& $\second{0.052}$ \\

AlphaSAGE
& $0.025$
& $0.043$
& $0.033$
& $0.050$\\

\rowcolor{oursbg}
\textbf{AlphaRJM (ours)}
& $\best{0.046}$
& $\best{0.066}$
& $\best{0.051}$
& $\best{0.065}$ \\

\bottomrule
\end{tabular}
\end{table}

\begin{figure}[t]
    \centering
    \includegraphics[width=\linewidth]{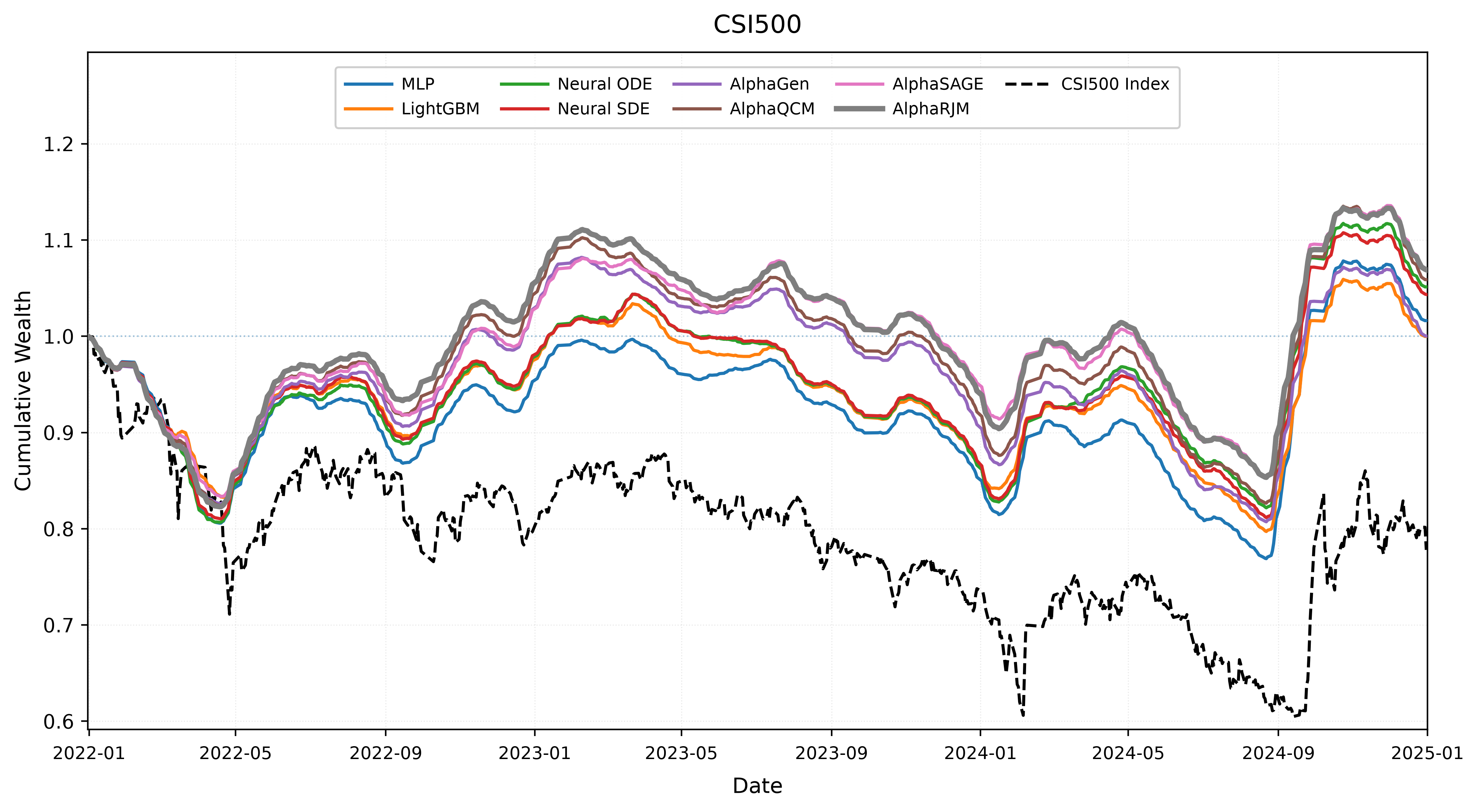}
    \caption{Cumulative wealth curves on CSI500 over the test period, averaged across random seeds 0--3.}
    \label{fig:csi500_cumulative_wealth}
\end{figure}

\subsection{Overall Performance and Robustness}

Table~\ref{tab:main_correlation_results} reports IC, ICIR, RIC, and RICIR
on CSI300, CSI500, and CSI800, averaged over four random seeds ($0$--$3$).
AlphaRJM achieves the best mean result in 10 of the 12 dataset--metric
comparisons and ranks second in the remaining two. It leads both IC and ICIR
on all three universes and all four metrics on CSI300 and CSI500; on CSI800,
it achieves the best IC and ICIR and ranks second on RIC and RICIR.
Despite stochastic variation across seeds, the method retains strong mean
performance overall, indicating that the gains are not driven by a single
favorable initialization.

Figure~\ref{fig:csi500_cumulative_wealth} provides a complementary view on CSI500.
AlphaRJM maintains the strongest cumulative-wealth trajectory for most of the
2022--2024 test period and recovers strongly after major market declines,
including the late-2024 rebound. The CSI500 index remains substantially below
the learned strategies over most of the evaluation period, supporting the
persistence of the predictive advantage observed in Table~\ref{tab:main_correlation_results}.

Controlled ablations in Appendix~\ref{app:ablation} isolate the contributions
of the SDE return critic, distributional supervision, and persistent memory.

\subsection{Long-Horizon Generalization}

At the $42$-trading-day forecasting horizon ($h=42$),
Table~\ref{tab:h42_ic_ric} reports the long-horizon results. AlphaRJM ranks first on both IC and RIC for CSI300 and CSI500,
achieving $(0.046,0.066)$ and $(0.051,0.065)$, respectively. Overall, AlphaRJM ranks first in four of the four long-horizon
dataset--metric comparisons.

Figure~\ref{fig:csi500_h42_cumulative_wealth} in Appendix~\ref{app:lg} shows that this advantage also extends
to cumulative wealth on CSI500: AlphaRJM maintains a clear lead for much of
the test period and reaches the highest cumulative-wealth level during the
late-2024 recovery. These results suggest that its predictive effectiveness
is not limited to the standard forecasting horizon.

\subsection{Sensitivity Analysis}
We evaluate AlphaRJM on CSI300 by varying the number of SDE particles, integration steps, energy-distance (CRPS-style) loss weight $\lambda_{\mathrm{ED}}$, and maximum diffusion scale $\sigma_{\max}$ (Fig.~\ref{fig:rjm_sensitivity}). Fewer particles generally yield stronger correlation metrics, while performance is relatively insensitive to the number of integration steps. The energy-distance weight introduces a trade-off across IC- and rank-based metrics, and moderate $\sigma_{\max}$ gives the most balanced results, whereas excessive diffusion degrades performance. Overall, AlphaRJM exhibits smooth and stable behavior across the tested configurations. The exact one-factor-at-a-time settings are given in Appendix~\ref{app:sensitivity_details}.

\begin{figure}[t]
    \centering
    \includegraphics[width=\linewidth]{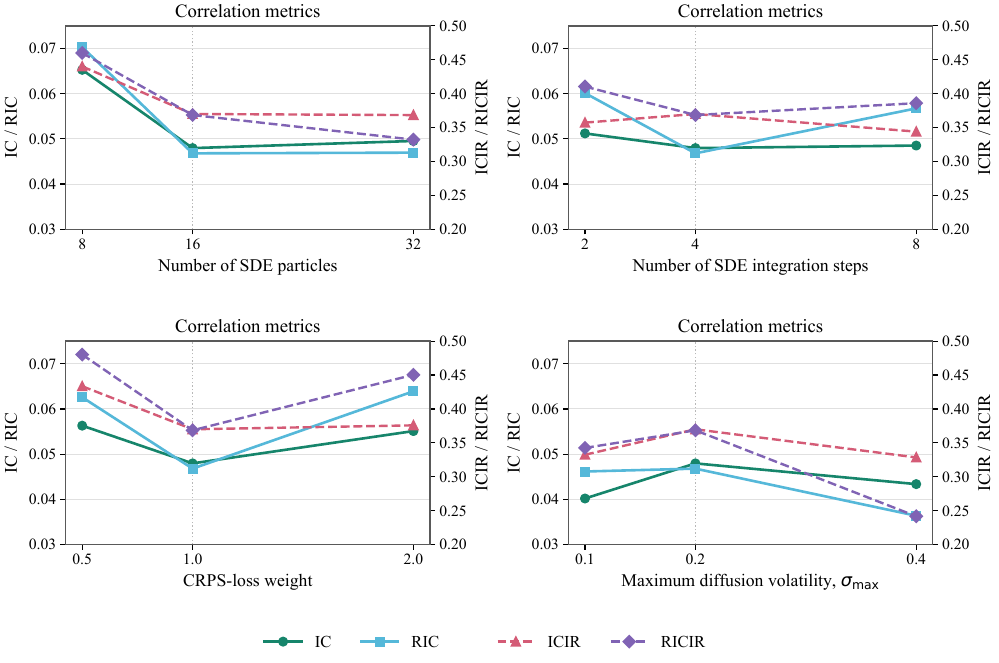}
    \caption{Sensitivity analysis of AlphaRJM on CSI300 with respect to the number of SDE particles, number of SDE integration steps, energy-distance (CRPS-style) loss weight $\lambda_{\mathrm{ED}}$, and maximum diffusion scale~$\sigma_{\max}$. The figure reports IC, RIC, ICIR, and RICIR, with the dotted vertical lines indicating the default configurations.}
    \label{fig:rjm_sensitivity}
\end{figure}

\section{Conclusion}

We introduced AlphaRJM, a formulaic alpha-discovery framework that
combines persistent terminal-evaluation history with stochastic
return-guided symbolic search. Reward-Jump Memory remains fixed during
token-level construction and updates only after terminal evaluations using
the realized pool reward and categorical outcome, allowing evaluation
history to persist across formula episodes even when the retained pool is
unchanged. An action-conditioned SDE return critic complements this memory
by generating particles that represent uncertain future discounted discovery
returns; their empirical mean and uncertainty guide action selection, while
their distribution is learned through a CRPS-style distributional Bellman
objective.

Across CSI300, CSI500, and CSI800, AlphaRJM achieves strong and stable
correlation-based performance relative to conventional, continuous-time, and
formulaic alpha-discovery baselines. The improvement extends to the
42-trading-day forecasting setting, while cumulative-wealth results provide
complementary evidence of out-of-sample effectiveness. Ablation experiments
support the complementary contributions of persistent evaluation history,
stochastic return modeling, and distributional supervision, and sensitivity
analysis shows stable behavior across the tested configurations. Overall,
the results indicate that preserving terminal reward--outcome history and
modeling future discovery returns distributionally provide complementary
signals for pool-dependent symbolic alpha discovery.

\bibliography{rjm}
\bibliographystyle{rjm}

\appendix

\clearpage
\section{The Use of Large Language Models (LLMs)}
\label{app:llm}

LLMs were used solely as auxiliary tools for language refinement, grammatical correction, improving clarity and presentation, and limited assistance in writing and debugging selected portions of the code. All methodological development, experiments, analyses, results, and scientific conclusions were independently conducted, carefully reviewed, and verified by the authors.

\section{AlphaRJM Implementation Details}
\label{app:implementation}

This appendix specifies the implementation corresponding to the
methodology in Section~\ref{sec:methodology}. Architectural dimensions,
symbolic-construction details, numerical SDE integration, exploration
schedules, replay, and optimization settings are given here to keep the
main text focused on the method itself.

\subsection{Architecture Overview}
\label{app:architecture_overview}

Figure~\ref{fig:alpharjm_architecture} summarizes the complete search loop.
At an exploitation step, the partial expression and retained alpha pool are
encoded, while the persistent memory $H_t$ directly conditions the
action-wise SDE return critic. At an exploration step, the
$\epsilon_n$-greedy branch samples directly from the currently available
action set $\mathcal A_t$. Terminal formula evaluations produce a pool
reward $r_t$ and outcome $o_t$, after which Reward-Jump Memory is updated
before construction of the next formula episode.

\begin{figure}[t]
    \centering
    \includegraphics[width=\linewidth]{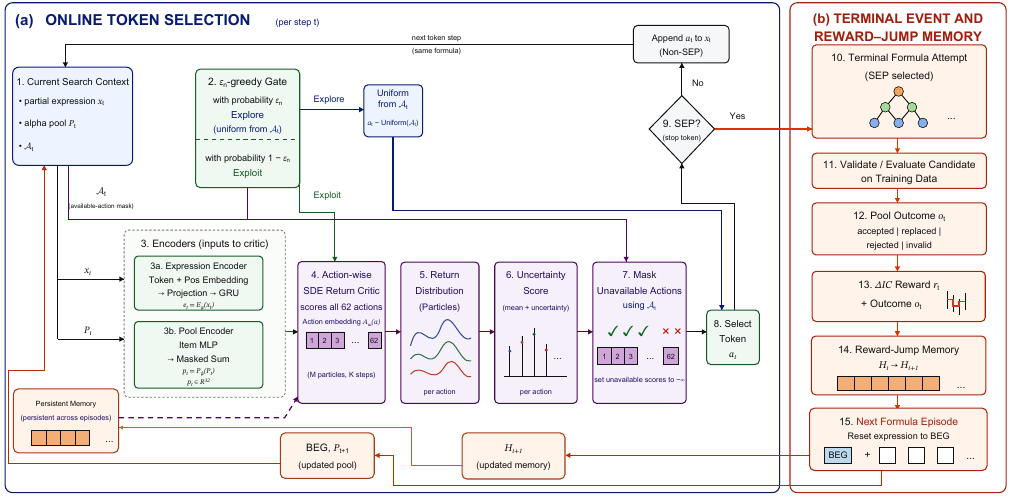}
    \caption{
    Architecture overview of AlphaRJM. The $\epsilon_n$-greedy policy
    either samples uniformly from $\mathcal A_t$, the action set exposed by the
    current construction mask, or exploits the action-conditioned SDE return
    critic using the encoded expression $e_t$, encoded pool $p_t$, persistent
    memory $H_t$, and candidate action. Terminal evaluations produce
    reward--outcome feedback $(r_t,o_t)$ that updates Reward-Jump Memory
    before the next formula episode.
    }
    \label{fig:alpharjm_architecture}
\end{figure}

\subsection{Symbolic Construction and Pool Evaluation}
\label{app:symbolic_search}

AlphaRJM uses the same expression grammar family as the controlled
formulaic-alpha benchmark. The selectable vocabulary contains 62 actions,
while the beginning-of-expression token (BEG) and padding token (PAD) are
input-only. Expressions contain at most
\[
    L_{\max}=20
\]
tokens in the benchmark state representation. The terminal separator action,
denoted SEP, completes the current formula attempt whenever selected.

The available action set $\mathcal A_t$ in
Section~\ref{sec:framework} is the set exposed by the benchmark construction
mask. The mask first applies the expression-builder legality rules. In
addition, we reproduce the released AlphaSAGE length-dependent early-stop
rule. Whenever SEP is currently available and the maximum length has not
been reached, all non-SEP actions are masked with probability
\begin{equation}
    p_{\mathrm{stop},t}
    =
    p_{\mathrm{mask}}
    \frac{\ell_t^{\mathrm{state}}}{L_{\max}},
    \label{eq:construction_mask_stop}
\end{equation}
where $\ell_t^{\mathrm{state}}$ is the current state length including BEG.
The reported runs use $p_{\mathrm{mask}}=1$. At the maximum length, SEP is
made available unconditionally so that the attempt terminates. The mask is
cached for the current state, ensuring that action selection and environment
execution use the same realization.

A terminal action increments the formula-episode counter and evaluates the
completed expression on the training data. A syntactically or numerically
invalid terminal attempt receives outcome $\mathrm{invalid}$ and reward
zero. Otherwise, the candidate is considered for inclusion in an alpha pool
of capacity 50. Pool quality is the weighted-ensemble training IC
$Q(\mathcal P)$, and the reward is exactly the marginal improvement defined
in Eq.~\ref{eq:pool_reward}.

Pool coefficients are jointly optimized using the controlled benchmark
mechanism. If the pool is not full, a valid inserted candidate receives
outcome $\mathrm{accepted}$. If the pool is full, the candidate and existing
members are jointly reweighted. If the candidate obtains the smallest
absolute coefficient, it receives outcome $\mathrm{rejected}$ and the
pre-candidate pool is restored exactly. Otherwise, the corresponding
existing member is removed and the candidate receives outcome
$\mathrm{replaced}$. A valid candidate whose IC statistics cannot be used
by the pool is likewise treated as $\mathrm{rejected}$.

Thus the reported environment realizes the four outcomes in
Eq.~\ref{eq:outcome_set}. A valid accepted or replaced candidate may
have a zero-valued incremental reward; no separate ``terminal-zero'' outcome
is emitted by the reported search environment.

\subsection{Expression and Pool Encoders}
\label{app:encoders}

The partial-expression encoder $E_{\phi}$ uses learned token and positional
embeddings, each of dimension 32. Their concatenation is projected to
dimension 32 and passed through a gated recurrent unit (GRU). The final valid
hidden state is the expression representation
\begin{equation}
    e_t\in\mathbb R^{32}.
\end{equation}

Completed formulas stored in the alpha pool reuse the token embeddings,
position embeddings, input projection, and GRU of the partial-expression
encoder. Their final hidden states are projected to 24-dimensional formula
embeddings. For pool member $j$, let
\[
    f_j\in\mathbb R^{24}
\]
denote this embedding.

The pool-item representation concatenates $f_j$ with five scalar quantities:
\begin{equation}
    \left[
        f_j;\,
        w_j;\,
        \operatorname{IC}_j;\,
        \ell_j;\,
        \rho_j;\,
        m_j
    \right],
    \label{eq:pool_item_input}
\end{equation}
where $w_j$ is the ensemble coefficient,
$\operatorname{IC}_j$ is the standalone training IC,
$\ell_j$ is the formula token length,
$\rho_j$ is the maximum absolute redundancy with the other retained
formulas, and $m_j\in\{0,1\}$ is the pool-slot occupancy mask.

Each occupied pool item is mapped to a 32-dimensional hidden representation.
The item representations are summed across the pool, yielding a
permutation-invariant aggregate, and a final projection produces
\begin{equation}
    p_t\in\mathbb R^{32}.
\end{equation}
A learned 32-dimensional embedding represents the empty pool.

The persistent memory dimension is
\begin{equation}
    H_t\in\mathbb R^8,
\end{equation}
and each action has a learned embedding
\begin{equation}
    A_{\omega}(a)\in\mathbb R^{16}.
\end{equation}
Consequently, Eq.~\ref{eq:sde_context} has dimension
\begin{equation}
    \dim(q_{t,a})
    =
    32+32+8+16
    =
    88.
\end{equation}

\subsection{Reward-Jump Memory Architecture}
\label{app:rjm_details}

At the beginning of the complete search stream,
Eq.~\ref{eq:rjm_init} maps the initial expression and pool
representations to $H_0\in\mathbb R^8$. In the reported implementation this
initialization is evaluated once without gradient tracking, and the resulting
$H_0$ is detached before the search begins. Consequently, no training loss is
backpropagated through the one-time initialization map $(W_0,b_0)$.

The implementation reserves five coordinates for categorical outcome
encoding. The four realized outcomes in
Eq.~\ref{eq:outcome_set} occupy four of these coordinates; the fifth
slot corresponds to a reserved terminal-zero code that is not emitted by the
reported environment. Thus
\begin{equation}
    [r_t;\operatorname{onehot}_{5}(o_t)]
    \in\mathbb R^6
    \label{eq:event_encoder_input}
\end{equation}
at a realized terminal event.

The event encoder implements
\[
\operatorname{LayerNorm}(6)
\rightarrow
\operatorname{Linear}(6,96)
\rightarrow
\operatorname{SiLU}
\rightarrow
\operatorname{Linear}(96,32)
\rightarrow
\tanh,
\]
producing
\[
    \chi_t\in\mathbb R^{32}.
\]

The eight-dimensional current memory and 32-dimensional event mark are
concatenated to form a 40-dimensional jump input. The context network is
\[
\operatorname{LayerNorm}(40)
\rightarrow
\operatorname{Linear}(40,96)
\rightarrow
\operatorname{SiLU}
\rightarrow
\operatorname{Linear}(96,96)
\rightarrow
\operatorname{SiLU}.
\]
Separate heads then produce
\[
    g_t\in(0,1)^8,
    \qquad
    \zeta_t\in(-1,1)^8.
\]
The experiments use
\[
    J_{\max}=0.02,
\]
so Eq.~\ref{eq:rjm_jump} becomes
\begin{equation}
    J_t
    =
    0.02\,g_t\odot\zeta_t,
\end{equation}
which guarantees
\[
    \|J_t\|_{\infty}\leq0.02.
\]

The event-driven jump depends on the current memory together with the
realized reward--outcome event mark. The expression and pool encodings do not
enter Eq.~\ref{eq:rjm_jump} directly. Reward-Jump Memory contains no
continuous drift, diffusion, or Brownian component; between terminal events,
Eq.~\ref{eq:rjm_update} leaves $H_t$ exactly unchanged.

\subsection{SDE Return Critic}
\label{app:sde_details}

The action-conditioned critic receives the 88-dimensional vector
$q_{t,a}$ from Eq.~\ref{eq:sde_context}. A two-layer SiLU network
with hidden width $d_{\mathrm{critic}}$ maps this vector to four scalar
heads:
\[
    z_{0,t,a},
    \qquad
    \mu_{t,a},
    \qquad
    \widehat\kappa_{t,a},
    \qquad
    \widehat\sigma_{t,a}.
\]

The positive mean-reversion rate in
Eq.~\ref{eq:return_sde} is parameterized as
\begin{equation}
    \kappa_{t,a}
    =
    \operatorname{softplus}
    (\widehat\kappa_{t,a})
    +
    10^{-4},
    \label{eq:kappa_parameterization}
\end{equation}
and the diffusion scale is
\begin{equation}
    \sigma_{t,a}
    =
    \sigma_{\min}
    +
    (\sigma_{\max}-\sigma_{\min})
    \operatorname{sigmoid}
    (\widehat\sigma_{t,a}),
    \label{eq:sigma_parameterization}
\end{equation}
with
\[
    \sigma_{\min}=10^{-3}.
\]

Eq.~\ref{eq:return_sde} is simulated using $K$ Euler--Maruyama
steps over $\tau\in[0,1]$. Define
\begin{equation}
    \Delta\tau=\frac{1}{K}.
    \label{eq:em_step_size}
\end{equation}
For fixed $(t,a)$, particle $m$ is initialized by
\[
    Z_0^{(m)}=z_{0,t,a}
\]
and evolves as
\begin{equation}
\begin{aligned}
    Z_{k+1}^{(m)}
    ={}&
    Z_k^{(m)}
    +
    \kappa_{t,a}
    \bigl(
        \mu_{t,a}-Z_k^{(m)}
    \bigr)
    \Delta\tau
    \\
    &+
    \sigma_{t,a}
    \sqrt{\Delta\tau}\,
    \varepsilon_k^{(m)},
\end{aligned}
\label{eq:em_update}
\end{equation}
where
\[
    \varepsilon_k^{(m)}
    \overset{\mathrm{i.i.d.}}{\sim}
    \mathcal N(0,1),
    \qquad
    k=0,\ldots,K-1,
    \quad
    m=1,\ldots,M.
\]
After $K$ integration steps,
\[
    Z_{t,a}^{(m)}:=Z_K^{(m)}
\]
is the terminal return particle appearing in
Eq.~\ref{eq:particles}. Independent Gaussian increments are used
across particles.

\subsection{Exploration Strategy}
\label{app:exploration}

The uncertainty coefficient in Eq.~\ref{eq:action_score} is linearly
annealed according to
\begin{equation}
    c_n
    =
    c_0
    +
    (c_N-c_0)
    \min\!\left(\frac{n}{N},1\right),
    \qquad
    c_0=1,
    \quad
    c_N=0.
    \label{eq:uncertainty_schedule}
\end{equation}

The $\epsilon_n$-greedy probability in
Eq.~\ref{eq:action_selection} follows
\begin{equation}
    \epsilon_n
    =
    \epsilon_0
    +
    (\epsilon_N-\epsilon_0)
    \min\!\left(\frac{n}{N},1\right),
    \qquad
    \epsilon_0=1,
    \quad
    \epsilon_N=0.05.
    \label{eq:epsilon_schedule}
\end{equation}
The reported runs use $N=10{,}000$ completed formula episodes.

On the random branch, an action is sampled uniformly from
$\mathcal A_t$ without invoking the SDE critic. On the exploitation branch,
the implementation vectorizes the SDE critic over all 62 selectable actions,
computes their particle scores, masks actions not contained in
$\mathcal A_t$ to $-\infty$, and selects the maximizer.

\subsection{Contiguous Replay and Optimization}
\label{app:replay}

The replay buffer stores primitive transitions together with the persistent
memory present immediately before each action and identifiers for the
corresponding current and next pool snapshots. Its capacity is 50,000
transitions.

Training samples contiguous sequences of length
\[
    L_{\mathrm{seq}}=16.
\]
The first
\[
    L_{\mathrm{burn}}=8
\]
positions serve as burn-in, after which the remaining eight positions
contribute directly to the SDE return-learning objective. Starting from the
stored pre-transition memory of each sampled sequence,
Eq.~\ref{eq:rjm_update} is unrolled chronologically so that later
learning positions use a memory trajectory reconstructed under the current
jump-network parameters.

For completeness, the implementation aggregates jump regularization over
all event-bearing positions of the full 16-step memory unroll, including
burn-in. Let $B_{\mathrm{batch}}$ be the replay batch size, let
$\delta_{b,k}$ and $J_{b,k}$ denote the terminal indicator and candidate jump
for sequence $b\in\{1,\ldots,B_{\mathrm{batch}}\}$ at replay position $k$, and
define
\[
    \mathcal K_{\mathrm{evt}}
    =
    \left\{
        k:
        \sum_{b=1}^{B_{\mathrm{batch}}}\delta_{b,k}>0
    \right\}.
\]
When $\mathcal K_{\mathrm{evt}}\neq\emptyset$, the implemented jump penalty is
\begin{equation}
    \mathcal L_{\mathrm{jump}}
    =
    \frac{1}{|\mathcal K_{\mathrm{evt}}|}
    \sum_{k\in\mathcal K_{\mathrm{evt}}}
    \frac{
        \sum_{b=1}^{B_{\mathrm{batch}}}
        \delta_{b,k}
        \|J_{b,k}\|_2^2
    }{
        \sum_{b=1}^{B_{\mathrm{batch}}}
        \delta_{b,k}
    }.
    \label{eq:jump_loss_exact}
\end{equation}
If no evaluation event occurs in the sampled unroll, this term is zero.

Replay optimization begins after 2,048 stored primitive transitions.
The replay batch size is
\[
    B_{\mathrm{batch}}=16,
\]
and one gradient update is triggered every four primitive actions.

The discount factor is
\[
    \gamma=0.99.
\]
Trainable online components are optimized jointly with AdamW and gradient
clipping at maximum norm 10. The target SDE return critic is synchronized
with the online critic every 50 optimizer updates.

For the default AlphaRJM configuration,
\begin{equation}
    \lambda_{\mathrm{ED}}=1,
    \qquad
    \lambda_{\mathrm{mean}}=0.1,
    \qquad
    \beta_J=10^{-4}.
    \label{eq:default_loss_weights}
\end{equation}

\subsection{AlphaRJM Pseudocode}
\label{app:alpharjm_pseudocode}

Algorithm~\ref{alg:alpharjm} summarizes the complete search and
optimization procedure using the notation introduced in
Section~\ref{sec:methodology}.

\begin{algorithm}[!t]
\caption{AlphaRJM}
\label{alg:alpharjm}
\small

\KwIn{Formula-episode budget $N$}
\KwOut{Final alpha pool $\mathcal P$}

Initialize $\mathcal P\leftarrow\emptyset$ and replay buffer
$\mathcal D\leftarrow\emptyset$\;

Initialize online model parameters and target SDE critic
$\bar\theta\leftarrow\theta$\;

Encode the initial context and initialize $H_0$
by Eq.~\ref{eq:rjm_init}\;

\While{completed formula episodes $<N$}{

    Compute
    $e_t\leftarrow E_{\phi}(x_t)$ and
    $p_t\leftarrow P_{\psi}(\mathcal P_t)$\;

    Obtain the current available action set $\mathcal A_t$\;

    \eIf{sample $\epsilon_n$-greedy exploration branch}{
        Sample $a_t$ uniformly from $\mathcal A_t$\;
    }{
        \ForEach{$a\in\mathcal A_t$}{
            Generate
            $\mathcal Z_{\theta}(s_t,a)$
            by Eqs.~\ref{eq:return_sde}--\ref{eq:particles}\;

            Compute $S_t(a)$
            by Eq.~\ref{eq:action_score}\;
        }

        $a_t\leftarrow
        \arg\max_{a\in\mathcal A_t}S_t(a)$\;
    }

    Execute $a_t$ and observe
    $(x_{t+1},\mathcal P_{t+1},r_t,\delta_t)$\;

    \eIf{$\delta_t=1$}{
        Observe terminal outcome $o_t$\;

        Compute event mark $\chi_t$
        by Eq.~\ref{eq:rjm_mark}\;

        Compute $J_t$
        by Eq.~\ref{eq:rjm_jump}\;

        $H_{t+1}\leftarrow H_t+J_t$\;

        Start the next formula episode while retaining $H_{t+1}$\;
    }{
        $H_{t+1}\leftarrow H_t$\;
    }

    Store the transition, pre-action memory $H_t$,
    and pool-snapshot identifiers in $\mathcal D$\;

    \If{replay conditions in Appendix~\ref{app:replay} are satisfied}{

        Sample contiguous sequences from $\mathcal D$\;

        Reconstruct the memory trajectory through burn-in\;

        \ForEach{learning position in the replay unroll}{

            Generate online particles
            $Z_t^{(m)}
            \sim
            \mathcal Z_{\theta}(s_t,a_t)$\;

            Select $a_{t+1}^{\star}$
            by Eq.~\ref{eq:next_action}\;

            Generate independent target particles
            $\widetilde Z_{t+1}^{(m)}
            \sim
            \mathcal Z_{\bar\theta}
            (s_{t+1},a_{t+1}^{\star})$\;

            Form Bellman particles
            $Y_t^{(m)}$
            by Eq.~\ref{eq:bellman_particles}\;
        }

        Compute the objective
        in Eq.~\ref{eq:total_loss}\;

        Update the online model parameters\;

        Periodically synchronize
        $\bar\theta\leftarrow\theta$\;
    }
}

\KwRet{$\mathcal P$}\;

\end{algorithm}

\subsection{AlphaRJM Hyperparameter Settings}
\label{app:rjm_presets}

Table~\ref{tab:rjm_presets} reports the AlphaRJM settings that vary across
the reported universes. Parameters not listed in the table follow the common
implementation settings above.

\begin{table}[H]
\centering
\caption{
AlphaRJM configurations used in the reported experiments. LR denotes the
AdamW learning rate, WD denotes weight decay, $d_{\mathrm{critic}}$ is the
hidden width of the SDE critic, $M$ is the number of return particles, and
$K$ is the number of Euler--Maruyama steps.
}
\label{tab:rjm_presets}

\scriptsize
\setlength{\tabcolsep}{2.4pt}
\renewcommand{\arraystretch}{1.10}

\resizebox{\linewidth}{!}{
\begin{tabular}{lcccccccc}
\toprule
\textbf{Preset}
& \textbf{LR}
& \textbf{WD}
& $\mathbf{d_{\mathrm{critic}}}$
& $\mathbf{M}$
& $\mathbf{K}$
& $\boldsymbol{\lambda_{\mathrm{ED}}}$
& $\boldsymbol{\lambda_{\mathrm{mean}}}$
& $\boldsymbol{\beta_J}$ \\
\midrule

CSI300/CSI500
& $3.0{\times}10^{-4}$
& $1.0{\times}10^{-4}$
& 64
& 16
& 4
& 1.0
& 0.1
& $1.0{\times}10^{-4}$ \\

CSI800/CSI1000
& $1.442{\times}10^{-4}$
& $2.144{\times}10^{-6}$
& 32
& 32
& 4
& 1.0
& 0.1
& $1.0{\times}10^{-4}$ \\

\bottomrule
\end{tabular}
}
\end{table}

The first preset is used for CSI300 and CSI500; the second is used for the
reported CSI800 experiment and the CSI1000 appendix experiment. Each
$h=42$ run reuses the corresponding universe's $h=20$ AlphaRJM
configuration rather than introducing long-horizon-specific model tuning.
All reported configurations use
\[
    \sigma_{\max}=0.20.
\]

\section{Sensitivity Configuration}

\label{app:sensitivity_details}

Sensitivity analysis is conducted on CSI300 with seed 0 using a
one-factor-at-a-time protocol. Starting from the default configuration, we
vary:
\begin{align}
    M &\in \{8,16,32\},\\
    K &\in \{2,4,8\},\\
    \lambda_{\mathrm{ED}}
      &\in \{0.5,1.0,2.0\},\\
    \sigma_{\max}
      &\in \{0.10,0.20,0.40\}.
\end{align}
The default operating point is
\[
    (M,K,\lambda_{\mathrm{ED}},\sigma_{\max})
    =
    (16,4,1.0,0.20).
\]
All other parameters remain fixed while one quantity is varied.

\section{Experiment Details}
\label{app:experiments}

This section complements the experimental summary in
Section~\ref{sec:experiments}. We first specify the common benchmark protocol,
then define the reported metrics, and finally describe the baseline
implementations and computational environment. AlphaRJM-specific architecture
and optimization settings are given in Appendix~\ref{app:implementation}, with
the universe-specific settings summarized in Table~\ref{tab:rjm_presets}.

\subsection{Benchmark Protocol}
\label{app:protocol}

The main benchmark uses CSI300, CSI500, and CSI800; CSI1000 is included as an
additional large-universe experiment in Appendix~\ref{app:csi1000_results}.
All methods use the same chronological partitions: January 1, 2010--December
31, 2020 for training, January 1--December 31, 2021 for validation, and January
1, 2022--December 31, 2024 for testing. Training and model selection therefore
use no observations from the test period.

Let $C_{i,d}$ denote the closing price of stock $i$ on trading date $d$. For a
forecasting horizon $h$, the prediction target is
\begin{equation}
    y_{i,d}^{(h)}
    =
    \frac{C_{i,d+h}}{C_{i,d}}-1,
    \label{eq:forward_return_target}
\end{equation}
implemented in the benchmark as
$\texttt{Ref(CLOSE,-h)/CLOSE-1}$. The main experiments use $h=20$, while the
long-horizon study changes only the target horizon to $h=42$.

The formula-discovery methods share the same expression grammar and an alpha
pool capacity of 50. AlphaQCM, AlphaSAGE, and AlphaRJM are trained for 10,000
terminal formula episodes. AlphaGen retains its native PPO step-based training
and uses 50,000 requested primitive-action steps; because PPO collects complete
2,048-step rollouts, the completed runs contain 51,200 executed primitive
actions. We preserve each method's native optimization unit rather than
rewriting its training procedure solely to express all budgets in the same
counter.

The main $h=20$ results are aggregated over random seeds $0,1,2,3$. The
$h=42$ comparison uses seed 0, and the CSI1000 appendix table uses seeds 0 and
1. Apart from these explicitly stated reporting choices, the dataset split,
target construction, universe definition, pool protocol, and downstream
evaluation are shared across compared methods whenever applicable.

\subsection{Evaluation Metrics}
\label{app:metrics}

Let $\widehat y_{i,d}^{(h)}$ denote the predicted alpha score corresponding to
the target in Eq.~\ref{eq:forward_return_target}. For each test date $d$, the
Information Coefficient (IC) and Rank Information Coefficient (RIC) are the
cross-sectional Pearson and rank correlations, respectively:
\begin{align}
    \operatorname{IC}_d
    &=
    \operatorname{Corr}_i
    \!\left(
        \widehat y_{i,d}^{(h)},
        y_{i,d}^{(h)}
    \right),\\
    \operatorname{RIC}_d
    &=
    \operatorname{Corr}_i
    \!\left(
        \operatorname{rank}(\widehat y_{i,d}^{(h)}),
        \operatorname{rank}(y_{i,d}^{(h)})
    \right).
\end{align}
Here $\operatorname{Corr}_i$ is computed across stocks at a fixed date,
whereas $\mathbb E_d$ and $\operatorname{Var}_d$ denote the empirical mean and
variance across test dates. The four reported statistics are
\begin{align}
    \operatorname{IC}
    &= \mathbb E_d[\operatorname{IC}_d],
    &
    \operatorname{ICIR}
    &=
    \frac{\mathbb E_d[\operatorname{IC}_d]}
         {\sqrt{\operatorname{Var}_d(\operatorname{IC}_d)}},\\
    \operatorname{RIC}
    &= \mathbb E_d[\operatorname{RIC}_d],
    &
    \operatorname{RICIR}
    &=
    \frac{\mathbb E_d[\operatorname{RIC}_d]}
         {\sqrt{\operatorname{Var}_d(\operatorname{RIC}_d)}}.
\end{align}
Higher values indicate stronger predictive performance or temporal stability.

\paragraph{Cumulative wealth.}
For the complementary portfolio visualization, stocks are ranked by their
predicted alpha score on each trading date $d$, and $\mathcal T_d$ denotes the
top $20\%$ of valid stocks. Following the evaluation implementation, the daily
contribution used for wealth accumulation is
\begin{equation}
    R_d
    =
    \frac{1}{h}\,
    \frac{1}{|\mathcal T_d|}
    \sum_{i\in\mathcal T_d} y_{i,d}^{(h)}.
    \label{eq:portfolio_return}
\end{equation}
Starting from $\mathcal W_0=1$, cumulative wealth is
\begin{equation}
    \mathcal W_d
    =
    \mathcal W_{d-1}(1+R_d)
    =
    \prod_{d'=1}^{d}(1+R_{d'}).
    \label{eq:cumulative_wealth}
\end{equation}
When multiple seeds are shown, the wealth trajectory is computed separately for
each seed and the plotted curve is their pointwise mean,
\begin{equation}
    \overline{\mathcal W}_d
    =
    \frac{1}{N_{\mathrm{seed}}}
    \sum_{\nu=1}^{N_{\mathrm{seed}}}
    \mathcal W_d^{(\nu)},
    \label{eq:mean_cumulative_wealth}
\end{equation}
where $N_{\mathrm{seed}}$ is the number of seeds and $\nu$ indexes the seed.
The same portfolio rule is used for all compared methods.

\subsection{Baseline Implementations}
\label{app:baselines}

The common protocol above is applied without redefining the internal learning
mechanism of each baseline. Our unified benchmark was assembled using the
public reference implementations of AlphaSAGE, AlphaQCM, and AlphaGen,
respectively:\footnote{\url{https://github.com/BerkinChen/AlphaSAGE},
\url{https://github.com/ZhuZhouFan/AlphaQCM}, and
\url{https://github.com/RL-MLDM/alphagen}.} For these formula-discovery
baselines, method-specific architectures and learner settings are retained from
the released implementations, while the shared data, target, grammar, pool,
and evaluation protocol follow Appendix~\ref{app:protocol}.

\paragraph{MLP and LightGBM.}
The multilayer perceptron and LightGBM baselines use the tabular feature
construction of the reference AlphaSAGE benchmark. They provide conventional
direct-prediction comparators under the same materialized features, labels,
and chronological partitions.

\paragraph{Neural ODE and Neural SDE.}
These are controlled direct-prediction baselines using the same materialized
features, labels, normalization, and train/validation/test partitions as the
other predictive models. Their continuous-time dynamics map market features to
the prediction target directly; they are therefore distinct from AlphaRJM's
SDE return critic, whose state models the distribution of future symbolic-search
returns.

\paragraph{AlphaGen.}
AlphaGen~\citep{yu2023generating} is the pool-synergy reinforcement-learning
baseline. We retain its PPO-based symbolic search and pool optimization from the
reference implementation; only the shared benchmark interface and the training
budget specified in Appendix~\ref{app:protocol} are imposed.

\paragraph{AlphaQCM.}
AlphaQCM~\citep{zhu2025alphaqcm} is the distributional reinforcement-learning
baseline. Its released IQN/QCM learner and uncertainty-guided action scoring
are retained, while it operates on the common symbolic environment and pool
protocol used in our benchmark.

\paragraph{AlphaSAGE.}
AlphaSAGE~\citep{chen2026alphasage} is the structure- and diversity-aware
formula-discovery baseline. We retain its graph-based expression representation
and GFlowNet-oriented search components from the public implementation under the
same benchmark protocol.

\subsection{Computational Environment}
\label{app:compute}

All reported experiments were executed locally on CPU using an Apple Mac mini
with an Apple M4 processor, 16\,GB unified memory, and a 256\,GB SSD. The
reproducibility environment uses Python~3.12.13, PyTorch~2.13.0,
Qlib~0.9.8.dev32, LightGBM~4.7.0, scikit-learn~1.9.0, NumPy~2.5.1, and
pandas~2.3.3.

\section{Additional Results}
\label{app:additional_results}

 \subsection{Ablation Study}
\label{app:ablation}

Table~\ref{tab:ablation_csi300_seed0} isolates the three components highlighted in the main text. The w/o SDE Return Critic variant replaces the SDE return-distribution critic with the QCM/IQN return critic while retaining event-conditioned memory, and produces the largest decline in both IC and ICIR. The w/o Persistent Memory variant removes recurrent use of historical memory and also lowers both metrics. The w/o Distributional Loss variant removes the energy-distance term while retaining the SDE dynamics and mean-calibration objective; it causes a moderate reduction in IC and a larger decrease in ICIR. Overall, the full model achieves the highest IC and ICIR, supporting the complementary roles of the SDE return critic, distributional supervision, and persistent memory.

\begin{table}[H]
\centering
\caption{
Ablation study of \textbf{AlphaRJM} on CSI300 using random seed 0.
Higher values are better. Best and second-best results are shown in
\textbf{bold} and \underline{underlined}, respectively.
}
\label{tab:ablation_csi300_seed0}

\footnotesize
\setlength{\tabcolsep}{7pt}
\renewcommand{\arraystretch}{1.12}

\begin{tabular}{lcc}
\toprule
\textbf{Model} & \textbf{IC} & \textbf{ICIR} \\
\midrule

\rowcolor{oursbg}
\textbf{AlphaRJM}
& $\mathbf{0.0480}$
& $\mathbf{0.3700}$ \\

w/o SDE Return Critic
& $0.0362$
& $0.2687$ \\

w/o Distributional Loss
& $\underline{0.0453}$
& $\underline{0.3085}$ \\

w/o Persistent Memory
& $0.0434$
& $0.3053$ \\

\bottomrule
\end{tabular}
\end{table}

\subsection{Additional Results on CSI1000}
\label{app:csi1000_results}

To further evaluate AlphaRJM on a broader equity universe, we report
additional results on CSI1000 using the same chronological data split,
prediction target, evaluation metrics, and controlled comparison protocol
described in Appendix~\ref{app:protocol}. Results are reported over
random seeds 0 and 1.
As shown in Table~\ref{tab:csi1000_results}, AlphaRJM achieves the
highest IC and RIC on CSI1000, reaching
$0.076\pm0.003$ and $0.102\pm0.009$, respectively.
It also obtains the second-highest RICIR
($0.648\pm0.013$), while AlphaQCM and LightGBM achieve the strongest
ICIR and RICIR, respectively. These results provide additional evidence
that the predictive advantage of AlphaRJM is retained when the evaluation
is extended to the larger CSI1000 universe, particularly for the
cross-sectional correlation metrics IC and RIC.

\begin{table}[H]
\centering
\caption{
Performance comparison on CSI1000 using correlation-based evaluation metrics.
Results are reported as mean $\pm$ standard deviation over two random seeds
(0--1). Higher values are better for all metrics. Best and second-best
performances are shown in \textbf{bold} and \underline{underlined},
respectively.
}
\label{tab:csi1000_results}

\footnotesize
\setlength{\tabcolsep}{4.0pt}
\renewcommand{\arraystretch}{1.12}

\begin{tabular}{lcccc}
\toprule
\textbf{Method}
& \textbf{IC}
& \textbf{ICIR}
& \textbf{RIC}
& \textbf{RICIR} \\
\midrule

MLP
& $0.049 \pm 0.004$
& $0.464 \pm 0.034$
& $0.063 \pm 0.008$
& $0.556 \pm 0.076$ \\

LightGBM
& $0.053 \pm 0.000$
& $\underline{0.556 \pm 0.006}$
& $0.063 \pm 0.000$
& $\mathbf{0.668 \pm 0.027}$ \\

Neural ODE
& $0.054 \pm 0.001$
& $0.513 \pm 0.020$
& $0.066 \pm 0.002$
& $0.585 \pm 0.008$ \\

Neural SDE
& $0.054 \pm 0.000$
& $0.510 \pm 0.000$
& $0.063 \pm 0.001$
& $0.574 \pm 0.006$ \\

AlphaGen
& $\underline{0.072 \pm 0.010}$
& $0.513 \pm 0.038$
& $\underline{0.089 \pm 0.011}$
& $0.627 \pm 0.035$ \\

AlphaQCM
& $0.070 \pm 0.008$
& $\mathbf{0.563 \pm 0.002}$
& $0.078 \pm 0.016$
& $0.610 \pm 0.080$ \\

AlphaSAGE
& $0.056 \pm 0.007$
& $0.517 \pm 0.032$
& $0.068 \pm 0.004$
& $0.579 \pm 0.018$ \\

\rowcolor{oursbg}
\textbf{AlphaRJM (ours)}
& $\mathbf{0.076 \pm 0.003}$
& $0.505 \pm 0.003$
& $\mathbf{0.102 \pm 0.009}$
& $\underline{0.648 \pm 0.013}$ \\

\bottomrule
\end{tabular}
\end{table}

\subsection{Long-Horizon Cumulative Wealth on CSI500.}
\label{app:lg}
Figure~\ref{fig:csi500_h42_cumulative_wealth} shows the cumulative wealth
trajectories on CSI500 for the $h=42$ forecasting horizon using random
seed~0. AlphaRJM maintains the highest wealth trajectory for most of the
test period and shows a strong recovery following the market decline in
mid-2024. By the end of the evaluation period, AlphaRJM remains above the
other learned strategies, while the CSI500 index remains substantially
below its initial level. Together with the corresponding IC and RIC
results, this visualization provides additional evidence that the
long-horizon predictive signals discovered by AlphaRJM translate into
persistent cross-sectional portfolio performance. Since the figure is
based on a single random seed, it is presented as a complementary
illustration rather than as a separate statistical comparison.

\begin{figure}[H]
    \centering
    \includegraphics[width=\linewidth]
    {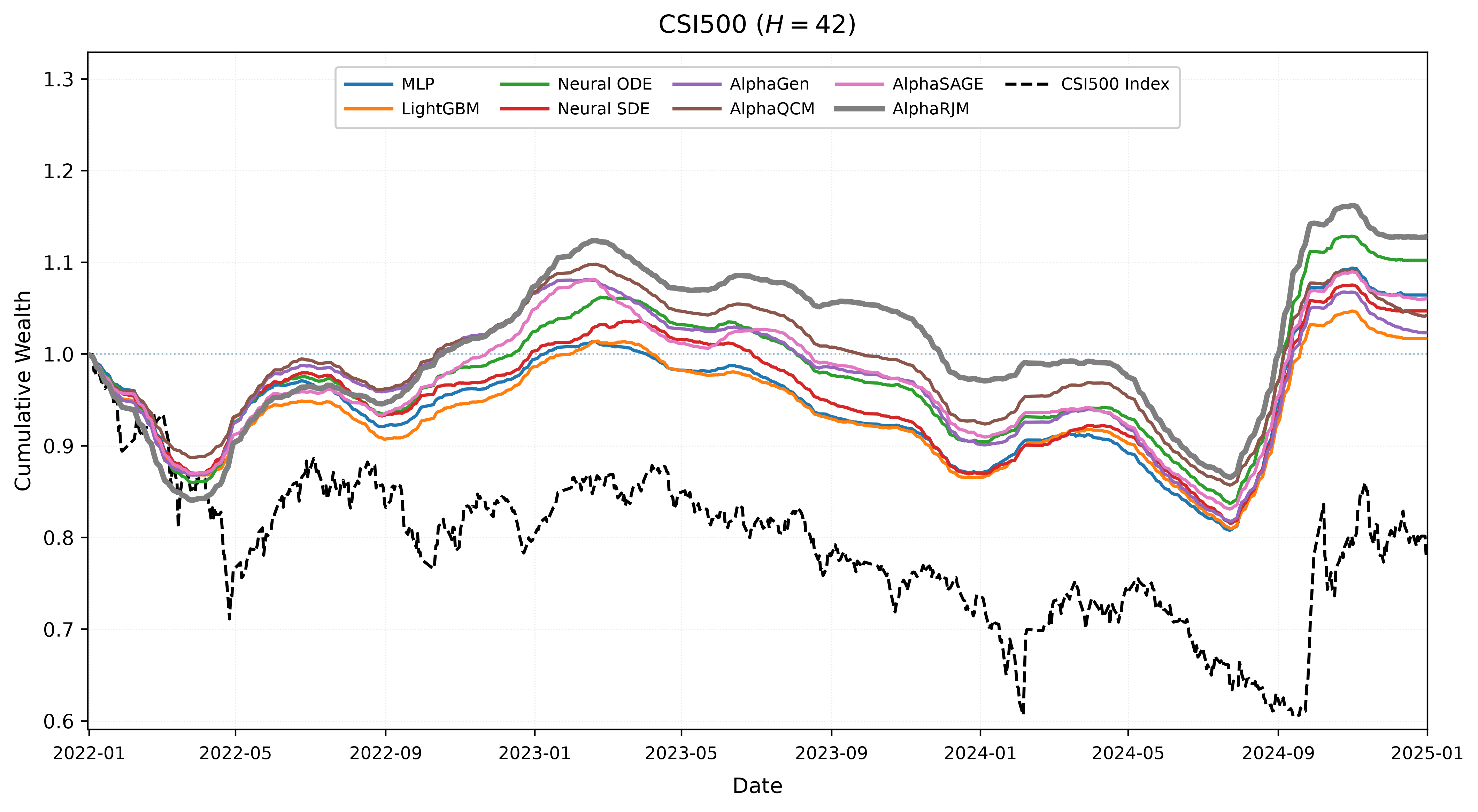}
    \caption{
    Cumulative wealth curves on CSI500 for the $h=42$ forecasting horizon
    using random seed 0.
    }
    \label{fig:csi500_h42_cumulative_wealth}
\end{figure}

\section{Theoretical Properties of AlphaRJM}
\label{app:theory}

This appendix establishes theoretical properties of the two central
components of AlphaRJM. We first show that Reward-Jump Memory is exactly
constant between terminal evaluation events and has bounded variation over
finite search trajectories. We then characterize the conditional law of the
action-conditioned SDE return critic using classical results for linear
stochastic differential equations
~\citep{oksendal2013stochastic,mao2007stochastic}. Finally, we derive the
moments of the finite-step Euler--Maruyama particles used by the
implementation~\citep{higham2001algorithmic} and establish consistency of
their empirical statistics using the strong law of large numbers
~\citep{durrett2019probability}.

\subsection{Event-Driven Memory Dynamics}

The following result formalizes the event-driven persistence of
Reward-Jump Memory and bounds its variation over a finite search trajectory.
Here $\|\cdot\|_{\infty}$ denotes the vector maximum norm.

\begin{proposition}[Event-Driven Persistence and Bounded Memory Variation]
\label{prop:rjm_persistence}

For Reward-Jump Memory defined by
Eqs.~\ref{eq:rjm_jump}--\ref{eq:rjm_update},
\begin{equation}
    H_{t+1}=H_t
    \qquad\text{whenever}\qquad
    \delta_t=0.
    \label{eq:rjm_no_event}
\end{equation}
Whenever $\delta_t=1$,
\begin{equation}
    \|H_{t+1}-H_t\|_{\infty}
    \leq
    J_{\max}.
    \label{eq:rjm_single_jump_bound}
\end{equation}
Moreover, for every integer $T\geq1$, the cumulative variation satisfies
\begin{equation}
    \sum_{t=0}^{T-1}
    \|H_{t+1}-H_t\|_{\infty}
    \leq
    J_{\max}
    \sum_{t=0}^{T-1}\delta_t .
    \label{eq:rjm_variation_bound}
\end{equation}
Consequently,
\begin{equation}
    \|H_T-H_0\|_{\infty}
    \leq
    J_{\max}
    \sum_{t=0}^{T-1}\delta_t .
    \label{eq:rjm_displacement_bound}
\end{equation}
Thus Reward-Jump Memory is exactly constant between terminal evaluation
events, while both its cumulative variation and net displacement over any
finite search prefix are controlled by the number of such events.
\end{proposition}

\begin{proof}
When $\delta_t=0$, Eq.~\ref{eq:rjm_update} gives
\[
    H_{t+1}=H_t,
\]
and therefore
\[
    H_{t+1}-H_t=0,
\]
which proves Eq.~\ref{eq:rjm_no_event}.

Now suppose that $\delta_t=1$. Then
Eq.~\ref{eq:rjm_update} gives
\[
    H_{t+1}-H_t=J_t.
\]
From Eq.~\ref{eq:rjm_jump},
\[
    J_t
    =
    J_{\max}\,g_t\odot\zeta_t.
\]
Since $g_t$ is obtained componentwise through a sigmoid function and
$\zeta_t$ through a hyperbolic tangent, for every coordinate $i$,
\[
    0<(g_t)_i<1,
    \qquad
    -1<(\zeta_t)_i<1.
\]
Hence
\[
    \|g_t\|_{\infty}\leq1,
    \qquad
    \|\zeta_t\|_{\infty}\leq1.
\]
Moreover,
\[
\begin{aligned}
    \|g_t\odot\zeta_t\|_{\infty}
    &=
    \max_i
    \left|
        (g_t)_i(\zeta_t)_i
    \right| \\
    &\leq
    \left(
        \max_i |(g_t)_i|
    \right)
    \left(
        \max_i |(\zeta_t)_i|
    \right) \\
    &=
    \|g_t\|_{\infty}
    \|\zeta_t\|_{\infty}
    \leq 1.
\end{aligned}
\]
Therefore,
\[
\begin{aligned}
    \|J_t\|_{\infty}
    &=
    J_{\max}
    \|g_t\odot\zeta_t\|_{\infty} \\
    &\leq
    J_{\max}.
\end{aligned}
\]
Since $H_{t+1}-H_t=J_t$ when $\delta_t=1$, we obtain
\[
    \|H_{t+1}-H_t\|_{\infty}
    \leq
    J_{\max},
\]
which proves Eq.~\ref{eq:rjm_single_jump_bound}.

Combining the cases $\delta_t=0$ and $\delta_t=1$ gives, for every $t$,
\[
    \|H_{t+1}-H_t\|_{\infty}
    \leq
    J_{\max}\delta_t.
\]
Summing over $t=0,\ldots,T-1$ yields
\[
\begin{aligned}
    \sum_{t=0}^{T-1}
    \|H_{t+1}-H_t\|_{\infty}
    &\leq
    J_{\max}
    \sum_{t=0}^{T-1}\delta_t,
\end{aligned}
\]
which establishes
Eq.~\ref{eq:rjm_variation_bound}.

Finally, by telescoping,
\[
    H_T-H_0
    =
    \sum_{t=0}^{T-1}
    (H_{t+1}-H_t).
\]
Applying the triangle inequality for the $\ell_\infty$ norm,
\[
\begin{aligned}
    \|H_T-H_0\|_{\infty}
    &=
    \left\|
        \sum_{t=0}^{T-1}
        (H_{t+1}-H_t)
    \right\|_{\infty} \\
    &\leq
    \sum_{t=0}^{T-1}
    \|H_{t+1}-H_t\|_{\infty} \\
    &\leq
    J_{\max}
    \sum_{t=0}^{T-1}\delta_t.
\end{aligned}
\]
Thus
Eq.~\ref{eq:rjm_displacement_bound}
follows.
\end{proof}

Proposition~\ref{prop:rjm_persistence} formalizes the separation between
termination of an individual formula episode and termination of the longer
memory stream. A terminal evaluation may modify $H_t$, but the resulting
state is retained while subsequent formulas are constructed.

\subsection{SDE Return Critic}

For a fixed search state $s_t$ and available action $a\in\mathcal A_t$,
the process in Eq.~\ref{eq:return_sde} is a scalar linear
mean-reverting diffusion. The following result specializes classical
linear-SDE theory~\citep{oksendal2013stochastic,mao2007stochastic}
to the return critic used by AlphaRJM.

\begin{theorem}[Well-Posedness and Conditional Law of the SDE Return Critic]
\label{thm:return_sde_law}

For any fixed search state $s_t$ and available action
$a\in\mathcal A_t$, the SDE return critic in
Eq.~\ref{eq:return_sde} admits a unique strong solution on
$\tau\in[0,1]$. The solution is
\begin{equation}
\begin{aligned}
    Z_{\tau}^{t,a}
    ={}&
    \mu_{t,a}
    +
    \bigl(
        z_{0,t,a}-\mu_{t,a}
    \bigr)
    e^{-\kappa_{t,a}\tau}
    \\
    &+
    \sigma_{t,a}
    \int_0^\tau
    e^{-\kappa_{t,a}(\tau-\tau')}
    \,dB_{\tau'} .
\end{aligned}
\label{eq:return_sde_exact_solution}
\end{equation}
Consequently,
\begin{equation}
\mathbb E
\left[
    Z_{\tau}^{t,a}
    \mid s_t,a
\right]
=
\mu_{t,a}
+
\bigl(
    z_{0,t,a}-\mu_{t,a}
\bigr)
e^{-\kappa_{t,a}\tau},
\label{eq:return_sde_conditional_mean}
\end{equation}
and
\begin{equation}
\operatorname{Var}
\left(
    Z_{\tau}^{t,a}
    \mid s_t,a
\right)
=
\frac{\sigma_{t,a}^{2}}
     {2\kappa_{t,a}}
\left(
    1-e^{-2\kappa_{t,a}\tau}
\right).
\label{eq:return_sde_conditional_variance}
\end{equation}
For every $\tau\in(0,1]$,
\begin{equation}
\begin{aligned}
Z_{\tau}^{t,a}\mid(s_t,a)
\sim
\mathcal N
\Bigg(
&
\mu_{t,a}
+
\bigl(
    z_{0,t,a}-\mu_{t,a}
\bigr)
e^{-\kappa_{t,a}\tau},
\\
&
\frac{\sigma_{t,a}^{2}}
     {2\kappa_{t,a}}
\left(
    1-e^{-2\kappa_{t,a}\tau}
\right)
\Bigg).
\end{aligned}
\label{eq:return_sde_conditional_law}
\end{equation}
At $\tau=0$, the process reduces to the deterministic initial value
$Z_0^{t,a}=z_{0,t,a}$. In particular, conditional on $(s_t,a)$,
$Z_{\tau}^{t,a}$ has finite moments of every finite order for all
$\tau\in[0,1]$.
\end{theorem}

\begin{proof}
Fix $(s_t,a)$. Conditional on this state--action pair, the critic outputs
$z_{0,t,a}$, $\mu_{t,a}$, $\kappa_{t,a}$, and $\sigma_{t,a}$ are fixed
finite scalars. By construction,
\[
    \kappa_{t,a}>0,
    \qquad
    0<\sigma_{\min}
    \leq
    \sigma_{t,a}
    \leq
    \sigma_{\max}<\infty.
\]

Define the drift and diffusion coefficients by
\[
    b_{t,a}(z)
    =
    \kappa_{t,a}(\mu_{t,a}-z),
    \qquad
    \varsigma_{t,a}(z)
    =
    \sigma_{t,a}.
\]
For any $z,z'\in\mathbb R$,
\[
\begin{aligned}
    |b_{t,a}(z)-b_{t,a}(z')|
    &=
    \kappa_{t,a}|z-z'|,
\end{aligned}
\]
so the drift is globally Lipschitz. Moreover,
\[
    |b_{t,a}(z)|
    \leq
    \kappa_{t,a}
    \bigl(
        |\mu_{t,a}|+|z|
    \bigr),
\]
and hence it satisfies a linear-growth bound. Since the diffusion
coefficient is constant in $z$,
\[
    |\varsigma_{t,a}(z)-\varsigma_{t,a}(z')|=0,
\]
and its boundedness implies the corresponding linear-growth condition.
Therefore, standard existence and uniqueness results for stochastic
differential equations imply that
Eq.~\ref{eq:return_sde} admits a unique strong solution on
$[0,1]$~\citep{oksendal2013stochastic,mao2007stochastic}.

To obtain the solution explicitly, rewrite
Eq.~\ref{eq:return_sde} as
\[
    dZ_{\tau}^{t,a}
    +
    \kappa_{t,a}Z_{\tau}^{t,a}\,d\tau
    =
    \kappa_{t,a}\mu_{t,a}\,d\tau
    +
    \sigma_{t,a}\,dB_{\tau}.
\]
For fixed $(s_t,a)$, $\kappa_{t,a}$ is constant with respect to the
internal diffusion coordinate $\tau$. Using the integrating factor
$e^{\kappa_{t,a}\tau}$ and It\^o's product rule,
\[
\begin{aligned}
d\!\left(
    e^{\kappa_{t,a}\tau}Z_{\tau}^{t,a}
\right)
&=
e^{\kappa_{t,a}\tau}\,dZ_{\tau}^{t,a}
+
\kappa_{t,a}
e^{\kappa_{t,a}\tau}
Z_{\tau}^{t,a}\,d\tau \\
&=
\kappa_{t,a}\mu_{t,a}
e^{\kappa_{t,a}\tau}\,d\tau
+
\sigma_{t,a}
e^{\kappa_{t,a}\tau}\,dB_{\tau},
\end{aligned}
\]
where the terms involving $Z_{\tau}^{t,a}$ cancel.

Integrating over $[0,\tau]$ and using
$Z_0^{t,a}=z_{0,t,a}$ gives
\[
\begin{aligned}
e^{\kappa_{t,a}\tau}Z_{\tau}^{t,a}
&=
z_{0,t,a}
+
\kappa_{t,a}\mu_{t,a}
\int_0^\tau
e^{\kappa_{t,a}\tau'}\,d\tau' \\
&\quad+
\sigma_{t,a}
\int_0^\tau
e^{\kappa_{t,a}\tau'}\,dB_{\tau'} \\
&=
z_{0,t,a}
+
\mu_{t,a}
\left(
    e^{\kappa_{t,a}\tau}-1
\right)
+
\sigma_{t,a}
\int_0^\tau
e^{\kappa_{t,a}\tau'}\,dB_{\tau'}.
\end{aligned}
\]
Multiplying by $e^{-\kappa_{t,a}\tau}$ and rearranging yields
\[
\begin{aligned}
Z_{\tau}^{t,a}
&=
\mu_{t,a}
+
\left(
    z_{0,t,a}-\mu_{t,a}
\right)
e^{-\kappa_{t,a}\tau} \\
&\quad+
\sigma_{t,a}
\int_0^\tau
e^{-\kappa_{t,a}(\tau-\tau')}
\,dB_{\tau'},
\end{aligned}
\]
which is
Eq.~\ref{eq:return_sde_exact_solution}.

Conditional on $(s_t,a)$, the integrand
$e^{-\kappa_{t,a}(\tau-\tau')}$ is deterministic. Hence
\[
    \int_0^\tau
    e^{-\kappa_{t,a}(\tau-\tau')}
    \,dB_{\tau'}
\]
is a Gaussian random variable with mean zero. Taking conditional
expectation in
Eq.~\ref{eq:return_sde_exact_solution} therefore gives
\[
    \mathbb E
    \left[
        Z_{\tau}^{t,a}
        \mid s_t,a
    \right]
    =
    \mu_{t,a}
    +
    \left(
        z_{0,t,a}-\mu_{t,a}
    \right)
    e^{-\kappa_{t,a}\tau},
\]
which proves
Eq.~\ref{eq:return_sde_conditional_mean}.

Since the deterministic terms do not contribute to the conditional
variance, It\^o isometry gives
\[
\begin{aligned}
\operatorname{Var}
\left(
    Z_{\tau}^{t,a}\mid s_t,a
\right)
&=
\sigma_{t,a}^{2}
\int_0^\tau
e^{-2\kappa_{t,a}(\tau-\tau')}
\,d\tau' \\
&=
\sigma_{t,a}^{2}
\int_0^\tau
e^{-2\kappa_{t,a}u}\,du \\
&=
\frac{\sigma_{t,a}^{2}}
     {2\kappa_{t,a}}
\left(
    1-e^{-2\kappa_{t,a}\tau}
\right),
\end{aligned}
\]
where $u=\tau-\tau'$. This proves
Eq.~\ref{eq:return_sde_conditional_variance}.

Finally, conditional on $(s_t,a)$,
Eq.~\ref{eq:return_sde_exact_solution} is the sum of deterministic
terms and a Gaussian It\^o integral. Therefore
$Z_{\tau}^{t,a}\mid(s_t,a)$ is Gaussian with the mean and variance derived
above, which establishes
Eq.~\ref{eq:return_sde_conditional_law}. For $\tau>0$, all finite
moments follow from Gaussianity, while at $\tau=0$ the process equals the
deterministic initial value $z_{0,t,a}$.
\end{proof}

Theorem~\ref{thm:return_sde_law} gives a direct interpretation of the
critic parameters. Conditional on $(s_t,a)$, the mean moves from
$z_{0,t,a}$ toward $\mu_{t,a}$ at a rate controlled by
$\kappa_{t,a}$, while $\sigma_{t,a}$ determines stochastic dispersion.
These quantities characterize the internal reinforcement-learning return
used for symbolic search and do not describe an asset-price process.

\subsection{Euler--Maruyama Particles and Action Scoring}

The implementation samples the return process using the Euler--Maruyama
scheme in Eq.~\ref{eq:em_update}
~\citep{higham2001algorithmic}. The next result characterizes the
finite-$K$ terminal particles actually used in the action score
of Eq.~\ref{eq:action_score}.

\begin{corollary}[Euler--Maruyama Particle Moments and Score Consistency]
\label{cor:em_particle_consistency}

Fix a search state $s_t$, an available action
$a\in\mathcal A_t$, formula episode $n$ (and hence $c_n$), and an
integer number of Euler--Maruyama integration steps $K\geq1$.
Under this fixed conditioning, consider independent Euler--Maruyama
trajectories generated according to Eq.~\ref{eq:em_update}, and let
$Z_{t,a}^{(m)}=Z_K^{(m)}$ denote the terminal state of trajectory $m$.
Then the terminal particles are independent and identically distributed,
and each is conditionally Gaussian with
\begin{equation}
\mathbb E
\left[
    Z_{t,a}^{(m)}
    \mid s_t,a
\right]
=
\mu_{t,a}
+
\bigl(
    z_{0,t,a}-\mu_{t,a}
\bigr)
\left(
    1-\frac{\kappa_{t,a}}{K}
\right)^K,
\label{eq:em_particle_mean}
\end{equation}
and
\begin{equation}
\operatorname{Var}
\left(
    Z_{t,a}^{(m)}
    \mid s_t,a
\right)
=
\frac{\sigma_{t,a}^{2}}{K}
\sum_{k=0}^{K-1}
\left(
    1-\frac{\kappa_{t,a}}{K}
\right)^{2k}.
\label{eq:em_particle_variance}
\end{equation}

For the particle statistics
$\overline Z_{t,a}$ and $\widehat V_{t,a}$ defined in
Eq.~\ref{eq:particle_statistics},
\begin{equation}
    \overline Z_{t,a}
    \xrightarrow[M\to\infty]{\mathrm{a.s.}}
    \mathbb E
    \left[
        Z_{t,a}^{(m)}
        \mid s_t,a
    \right],
    \label{eq:particle_mean_consistency}
\end{equation}
and
\begin{equation}
    \widehat V_{t,a}
    \xrightarrow[M\to\infty]{\mathrm{a.s.}}
    \operatorname{Var}
    \left(
        Z_{t,a}^{(m)}
        \mid s_t,a
    \right).
    \label{eq:particle_variance_consistency}
\end{equation}
Consequently,
\begin{equation}
\begin{aligned}
S_t(a)
\xrightarrow[M\to\infty]{\mathrm{a.s.}}
{}&
\mu_{t,a}
+
\bigl(
    z_{0,t,a}-\mu_{t,a}
\bigr)
\left(
    1-\frac{\kappa_{t,a}}{K}
\right)^K
\\
&+
c_n
\Bigg[
    \frac{\sigma_{t,a}^{2}}{K}
    \sum_{k=0}^{K-1}
    \left(
        1-\frac{\kappa_{t,a}}{K}
    \right)^{2k}
    +
    \varepsilon_{\mathrm{num}}
\Bigg]^{1/2}.
\end{aligned}
\label{eq:action_score_limit}
\end{equation}
\end{corollary}
\begin{proof}
Fix $(s_t,a)$ and $K$. Using $\Delta\tau=1/K$ in
Eq.~\ref{eq:em_update}, the Euler--Maruyama recursion becomes
\[
    Z_{k+1}^{(m)}
    =
    Z_k^{(m)}
    +
    \frac{\kappa_{t,a}}{K}
    \left(
        \mu_{t,a}-Z_k^{(m)}
    \right)
    +
    \frac{\sigma_{t,a}}{\sqrt K}
    \varepsilon_k^{(m)},
\]
where
$\varepsilon_k^{(m)}\stackrel{\mathrm{i.i.d.}}{\sim}\mathcal N(0,1)$.
Subtracting $\mu_{t,a}$ from both sides gives
\[
    Z_{k+1}^{(m)}-\mu_{t,a}
    =
    \left(
        1-\frac{\kappa_{t,a}}{K}
    \right)
    \left(
        Z_k^{(m)}-\mu_{t,a}
    \right)
    +
    \frac{\sigma_{t,a}}{\sqrt K}
    \varepsilon_k^{(m)}.
\]
For brevity, define
\[
    \rho_{t,a}
    :=
    1-\frac{\kappa_{t,a}}{K}.
\]
Iterating the recursion from
$Z_0^{(m)}=z_{0,t,a}$ yields
\[
\begin{aligned}
    Z_K^{(m)}-\mu_{t,a}
    &=
    \rho_{t,a}^{K}
    \left(
        z_{0,t,a}-\mu_{t,a}
    \right) \\
    &\quad+
    \frac{\sigma_{t,a}}{\sqrt K}
    \sum_{k=0}^{K-1}
    \rho_{t,a}^{\,K-1-k}
    \varepsilon_k^{(m)}.
\end{aligned}
\]
Equivalently,
\[
\begin{aligned}
    Z_K^{(m)}
    &=
    \mu_{t,a}
    +
    \rho_{t,a}^{K}
    \left(
        z_{0,t,a}-\mu_{t,a}
    \right) \\
    &\quad+
    \frac{\sigma_{t,a}}{\sqrt K}
    \sum_{k=0}^{K-1}
    \rho_{t,a}^{\,K-1-k}
    \varepsilon_k^{(m)}.
\end{aligned}
\]

Conditional on $(s_t,a)$, all coefficients in this expression are fixed.
Since the variables $\varepsilon_k^{(m)}$ are independent standard
Gaussians, their weighted sum is Gaussian. Hence
$Z_K^{(m)}\mid(s_t,a)$ is Gaussian. Moreover,
$\mathbb E[\varepsilon_k^{(m)}]=0$, so
\[
\begin{aligned}
    \mathbb E
    \left[
        Z_K^{(m)}
        \mid s_t,a
    \right]
    &=
    \mu_{t,a}
    +
    \rho_{t,a}^{K}
    \left(
        z_{0,t,a}-\mu_{t,a}
    \right) \\
    &=
    \mu_{t,a}
    +
    \left(
        1-\frac{\kappa_{t,a}}{K}
    \right)^K
    \left(
        z_{0,t,a}-\mu_{t,a}
    \right),
\end{aligned}
\]
which proves Eq.~\ref{eq:em_particle_mean}.

The deterministic terms do not contribute to the conditional variance.
Using independence and
$\operatorname{Var}(\varepsilon_k^{(m)})=1$,
\[
\begin{aligned}
    \operatorname{Var}
    \left(
        Z_K^{(m)}
        \mid s_t,a
    \right)
    &=
    \frac{\sigma_{t,a}^{2}}{K}
    \sum_{k=0}^{K-1}
    \rho_{t,a}^{\,2(K-1-k)} \\
    &=
    \frac{\sigma_{t,a}^{2}}{K}
    \sum_{j=0}^{K-1}
    \rho_{t,a}^{\,2j} \\
    &=
    \frac{\sigma_{t,a}^{2}}{K}
    \sum_{j=0}^{K-1}
    \left(
        1-\frac{\kappa_{t,a}}{K}
    \right)^{2j},
\end{aligned}
\]
where the second equality follows from the reindexing
$j=K-1-k$. This establishes
Equation~(\ref{eq:em_particle_variance}).

For fixed $(s_t,a)$ and $K$, the terminal particles
$Z_{t,a}^{(1)},\ldots,Z_{t,a}^{(M)}$ are conditionally independent and
identically distributed because each is generated from the same
state--action-conditioned critic parameters using an independent noise
trajectory. Their conditional distribution is Gaussian by the preceding
derivation, and therefore
\[
    \mathbb E
    \left[
        \left(
            Z_{t,a}^{(m)}
        \right)^2
        \mid s_t,a
    \right]
    <\infty.
\]
Hence, by the strong law of large numbers
~\citep{durrett2019probability},
\[
    \overline Z_{t,a}
    =
    \frac{1}{M}
    \sum_{m=1}^{M}
    Z_{t,a}^{(m)}
    \xrightarrow{\mathrm{a.s.}}
    \mathbb E
    \left[
        Z_{t,a}^{(m)}
        \mid s_t,a
    \right],
\]
which proves
Eq.~\ref{eq:particle_mean_consistency}.

Since the conditional second moment is finite, the strong law also applies
to $\left(Z_{t,a}^{(m)}\right)^2$, giving
\[
    \frac{1}{M}
    \sum_{m=1}^{M}
    \left(
        Z_{t,a}^{(m)}
    \right)^2
    \xrightarrow{\mathrm{a.s.}}
    \mathbb E
    \left[
        \left(
            Z_{t,a}^{(m)}
        \right)^2
        \mid s_t,a
    \right].
\]
From Eq.~\ref{eq:particle_statistics},
\[
    \widehat V_{t,a}
    =
    \frac{1}{M}
    \sum_{m=1}^{M}
    \left(
        Z_{t,a}^{(m)}
    \right)^2
    -
    \overline Z_{t,a}^{\,2}.
\]
Since
\[
    \overline Z_{t,a}^{\,2}
    \xrightarrow{\mathrm{a.s.}}
    \left(
        \mathbb E
        \left[
            Z_{t,a}^{(m)}
            \mid s_t,a
        \right]
    \right)^2,
\]
it follows that
\[
\begin{aligned}
    \widehat V_{t,a}
    \xrightarrow{\mathrm{a.s.}}
    {}&
    \mathbb E
    \left[
        \left(
            Z_{t,a}^{(m)}
        \right)^2
        \mid s_t,a
    \right] \\
    &-
    \left(
        \mathbb E
        \left[
            Z_{t,a}^{(m)}
            \mid s_t,a
        \right]
    \right)^2 \\
    ={}&
    \operatorname{Var}
    \left(
        Z_{t,a}^{(m)}
        \mid s_t,a
    \right),
\end{aligned}
\]
which proves
Eq.~\ref{eq:particle_variance_consistency}.

Finally, for fixed $n$, define
\[
    f(z,v)
    =
    z+c_n\sqrt{v+\varepsilon_{\mathrm{num}}}.
\]
Since $\varepsilon_{\mathrm{num}}>0$, the mapping
$f:\mathbb R\times[0,\infty)\to\mathbb R$ is continuous. Therefore,
by Eqs.~\ref{eq:particle_mean_consistency} and
(\ref{eq:particle_variance_consistency}) and continuity of $f$,
\[
\begin{aligned}
    S_t(a)
    &=
    f\!\left(
        \overline Z_{t,a},
        \widehat V_{t,a}
    \right) \\
    &\xrightarrow{\mathrm{a.s.}}
    \mathbb E
    \left[
        Z_{t,a}^{(m)}
        \mid s_t,a
    \right]
    +
    c_n
    \sqrt{
        \operatorname{Var}
        \left(
            Z_{t,a}^{(m)}
            \mid s_t,a
        \right)
        +
        \varepsilon_{\mathrm{num}}
    }.
\end{aligned}
\]
Using Eq.~\ref{eq:action_score}, this is precisely
Eq.~\ref{eq:action_score_limit}.
\end{proof}

Corollary~\ref{cor:em_particle_consistency} shows that, for fixed $K$, the empirical particle mean and
variance consistently estimate the corresponding moments of the finite-step
Euler--Maruyama return distribution. Their Monte Carlo estimation errors
vanish almost surely as $M\to\infty$, while the underlying finite-$K$
discretized return process remains unchanged.

\end{document}

%% file: math_commands.tex
\usepackage{amsmath,amsfonts,bm}

\def\eqref#1{equation~\ref{#1}}
\def\1{\bm{1}}

\DeclareMathAlphabet{\mathsfit}{\encodingdefault}{\sfdefault}{m}{sl}
\SetMathAlphabet{\mathsfit}{bold}{\encodingdefault}{\sfdefault}{bx}{n}